\documentclass[10pt,journal]{IEEEtran}
\IEEEoverridecommandlockouts
\makeatletter
\def\endthebibliography{%
  \def\@noitemerr{\@latex@warning{Empty `thebibliography' environment}}%
  \endlist
}
\makeatother

\usepackage{cite}
\ifCLASSINFOpdf
\usepackage[pdftex]{graphicx}
\else
\fi
\usepackage{enumitem}
\usepackage{balance}
\usepackage{amsmath,amssymb,amsfonts}
\usepackage{textcomp}
\usepackage{bbm}
\usepackage{xcolor}
\usepackage{epsfig}
\usepackage{epstopdf}
\usepackage{algorithm}
\usepackage{algpseudocode}
\usepackage{graphicx}
\usepackage{subfig}
\usepackage{varwidth}
\usepackage{booktabs}
\usepackage{multirow}
\usepackage{ntheorem}
\usepackage{diagbox}
\usepackage{makecell}
\usepackage[font=footnotesize]{caption}
\usepackage[colorlinks,
            linkcolor=black,
            anchorcolor=black,
            citecolor=black]{hyperref}

\ifodd 1
\newcommand{\com}[1]{{\color{black}#1 }}
\else

\newcommand{\com}[1]{}
\fi

\newtheorem{theorem}{\textbf{Theorem}}
\newtheorem*{proof}{Proof}

\newtheorem{remark}{Remark}

\newtheorem{corollary}{\textbf{Corollary}}

\begin{document}
\title{From Learning to Analytics: Improving Model Efficacy with Goal-Directed Client Selection}
\author{\IEEEauthorblockN{Jingwen Tong, ~\IEEEmembership{Member,~IEEE},
Zhenzhen Chen,
Liqun Fu, ~\IEEEmembership{Senior Member,~IEEE},
Jun Zhang,  ~\IEEEmembership{Fellow,~IEEE},
 and
Zhu Han, ~\IEEEmembership{Fellow,~IEEE}}
\thanks{\emph{Corresponding author: Liqun Fu.}
This work was partly presented at IEEE ICC 2022 \cite{chen2022over}.
J. Tong and J. Zhang are with the Department of Electronic and Computer Engineering, The Hong Kong University of Science and Technology, Hong Kong, China (e-mails: eejwentong@ust.hk; eejzhang@ust.hk).
Z. Chen and L. Fu  are with the School of Informatics, Xiamen University, Xiamen 361005, China
(e-mails: zzchen@163.com; liqun@xmu.edu.cn).
 Z. Han is with the Department of Electrical and Computer Engineering at the University of Houston, Houston, TX 77004 USA, and also with the Department of Computer Science and Engineering, Kyung Hee University, Seoul 446-701, South Korea  (e-mail: zhan2@uh.edu).
}
}	

\IEEEtitleabstractindextext{
\begin{abstract}
Federated learning (FL) is an appealing paradigm for learning a global model among distributed clients while preserving data privacy.
Driven by the demand for high-quality user experiences, evaluating the well-trained global model after the FL process is crucial.
In this paper, we propose a closed-loop model analytics framework that allows for effective evaluation of the trained global model using clients' local data.
To address the challenges posed by system and data heterogeneities in the FL process,
we study a \textit{goal-directed} client selection problem based on the model analytics framework by selecting a subset of clients for the model training.
This problem is formulated as a stochastic multi-armed bandit (SMAB) problem.
We first put forth a quick initial upper confidence bound (Quick-Init UCB) algorithm to solve this SMAB problem under the federated analytics (FA) framework.
Then, we further propose a belief propagation-based UCB (BP-UCB) algorithm under the democratized analytics (DA) framework.
Moreover, we derive two regret upper bounds for the proposed algorithms, which increase logarithmically over the time horizon.
The numerical results demonstrate that the proposed algorithms achieve nearly optimal performance, with a gap of less than $1.44\%$ and $3.12\%$ under the FA and DA frameworks, respectively.
\end{abstract}

}

\maketitle

\IEEEdisplaynontitleabstractindextext

\IEEEpeerreviewmaketitle

\if CLASSOPTIONcompsoc
\IEEEraisesectionheading{\section{Introduction}\label{sec:introduction}}
\else
\section{Introduction}
\label{sec:introduction}
\fi
With the exponential advancement in storage capability and computing power, coupled with the harsh requirement of data security and privacy, it is desirable to process the high-volume data locally by clients rather than sending them to a central server \cite{abbas2017mobile}.
This emerging paradigm has given rise to the concept of federated learning (FL).
FL enables the collaborative learning of a global model across distributed local clients without the need for centralized data collection while preserving the privacy of clients' data \cite{ li2021blockchain}.
In recent years, FL has garnered significant attention in various communication networks, including but not limited to mobile edge computing networks, Internet-of-Things networks, and vehicle networks.

In the above applications, however, the efficiency of the training process in the FL is not the ultimate goal.
Instead, driven by the demand for high-quality user experiences, optimizing the efficacy of the trained global model by iteratively fine-tuning the training process is essential.
For instance,  Google's Gboard \cite{ramage2020federated} employs the FL to train a deep learning model for next-word prediction.
To improve the prediction accuracy, they evaluate the trained global model periodically by adjusting the training process based on the clients' testing results, i.e., the accuracy of the prediction model.
Therefore, a model analytics framework is essential for evaluating the trained global model.

In this paper, we investigate two types of model analytics frameworks: federated analytics (FA) \cite{ramage2020federated} and democratized analytics (DA) \cite{nguyen2021distributed}.
The FA framework represents a novel distributed computing paradigm for data analytics applications with privacy concerns.
It evaluates the trained global model using clients' local datasets without data collection in the central server.
The evaluation results, named test opinions, are then uploaded to the central server for model adjustment.
In contrast, the DA framework operates on a decentralized model, where clients self-organize or form well-connected networks to exchange their evaluation results.
To accommodate this decentralized feature, we will employ the belief propagation (BP) algorithm to facilitate message exchange between clients.

It is important to note that both the FA and DA frameworks connect the trained global model and the original dataset. In this way, the evaluation results obtained through these frameworks can effectively guide the FL training process in subsequent rounds. This capability enables the development of two additional frameworks: FL\&FA and FL\&DA. These frameworks incorporate a closed-loop mechanism to
iteratively adjust the FL training process.
Based on the proposed frameworks,
we study a \textit{goal-directed} client selection problem by selecting a subset of clients for the model training in FL.
The \textit{goal-directed} client selection mechanism not only can directly optimize the efficacy of the trained global model using the original dataset but also address the challenges of system and data heterogeneity in the FL training process.

The objective of the client selection problem is to find an optimal subset of clients for model training,
with the aim of maximizing the clients' average opinions.
This problem can be formulated as an online learning problem,
where the central server need to balance the \textit{exploitation} and \textit{exploration} (EE) dilemma to achieve optimal results.
On the one hand, the central server needs to exploit the current best client set as many times as possible to maximize its total rewards;
On the other hand, it also requires exploring other client sets as often as possible to avoid missing the optimal solution.
To balance this dilemma, we further model this online learning problem as a stochastic multi-armed bandit (SMAB) problem \cite{bubeck2012regret}.
In the FL\&FA framework, the player and the arms in the SMAB problem are the central server and the subsets of clients, respectively;
while the player and the arms are both clients in the FL\&DA framework.
Then, we propose two upper confidence bound (UCB)-based algorithms \cite{auer2002finite}  to solve this SMAB problem under the FL\&FA and  FL\&DA frameworks.

\subsection{Related Works}
FA was proposed by Google in 2020 \cite{ramage2020federated}  for data analytics applications, which can be considered into a generalized perspective and a specialized perspective \cite{wang2021federated}.
From the generalized perspective, FA contains all computing tasks that draw opinions from data and thus includes the FL.
From the specialized perspective, FA refers to the model inference phase of machine learning and thus is the sequel of the FL.
Wang \textit{et al.} \cite{wang2022secure} studied the secure trajectory publication problem in untrusted environments using the FA approach from the generalized perspective.
Similarly, Chen \textit{et al.} \cite{chen2021digital}  investigated a digital twin-assisted federated distribution discovery problem using the FA approach.
In practice, clients may prefer to share information with their neighbors rather than upload it to the central server due to privacy concerns.
This results in the DA framework, where clients are self-organized or well-connected.
Pandey \textit{et al.} \cite{pandey2021edge} proposed an edge-assisted democratized learning
and analytics framework to improve the generalization capability of the trained global model.
In this paper, we consider the FA and DA frameworks from the specialized perspective, i.e., to evaluate the trained global model after the FL.

Client selection problem for FL has been studied in \cite{nishio2019client, yang2019scheduling,  perazzone2022communication}  by considering the system heterogeneity, which is also known as the straggler problem \cite{chai2020tifl}.
Nishio and Yonetani \cite{nishio2019client} studied the client selection problem in a mobile edge computing system by choosing a subset of clients with better computation capabilities and channel conditions for model training.
Meanwhile, several scheduling methods, such as random selection, proportional fair selection, and round-robbin selection, have been investigated in \cite{yang2019scheduling} to conquer the limited communication resource as the high dimensionality of model parameters and a larger number of clients.
By considering the time-varying and deep-fading channel models, ref. \cite{ perazzone2022communication} proposed to select a subset of clients with a better channel condition to perform the FL.

However, several works investigate the client selection problem by considering the data heterogeneity, usually referring to the data size and data quality heterogeneities.
Wang \textit{et al.} \cite{wang2020optimizing}  studied the client selection problem with non-independent and identically distributed (non-i.i.d.) client data distributions and proposed two selection algorithms to mitigate the straggler effect in the FL.
Cho \textit{et al.} \cite{cho2020client} considered an unbiased selection method to address this problem by selecting the clients with more training loss. This indicates that clients with higher local training loss should be selected to achieve faster convergence.
\com{Deng \textit{et al.}  \cite{deng2022blockchain} proposed a client scheduling algorithm in the blockchain-assisted FL over wireless channel by considering the data heterogeneity.}
In addition, AbdulRahman \textit{et al.} \cite{abdulrahman2020fedmccs} studied the data size-based client selection to enhance the model training efficiency.

Recently, the MAB-based client selection algorithms are widely considered in FL \cite{zhu2022online, shi2022vfedcs, huang2020efficiency, huang2022contextfl}.
\com{As one of the bandit models for sequential decision-making problems with stochastic processes \cite{bubeck2012regret},
SMAB attracts great attention in resource scheduling problems.
It is an efficient framework to balance the EE dilemma whose solutions have low computational complexity, are easy to implement, and have strict theoretical performance guarantees (i.e., the regret bound).}
Zhu \textit{et al.} \cite{zhu2022online} studied the client selection for asynchronous FL with a fairness guarantee and proposed a UCB-based algorithm where only one client is selected at each time slot.
Similarly, Huang \textit{et al.} \cite{huang2020efficiency} proposed two UCB-based client selection algorithms in the FL accounting for the data size and quality heterogeneity, respectively.
However, they must select a subset of clients rather than just one client at each time slot.
In fact, some side information can be used to guide the client selection process.
Huang \textit{et al.} \cite{huang2022contextfl} developed a contextual MAB-based client selection algorithm in mobile edge computing systems where the estimated computation and communication time is regarded as the context.

However, there are two drawbacks in the above client selection works.
First, they often assume a causal relationship between the FL training efficiency  (\textit{e.g.}, convergence rate and accuracy) and the pre-determined factors (\textit{e.g.}, channel condition, computation power, data size, and data quality).
Under this assumption, targeting one factor may result in a suboptimal solution;
while considering all factors that suffer from high system complexity since they require some global information.
In fact, the performance of the trained global model should be established on the real dataset rather than the pre-determined causal relationship \cite{li2022pyramidfl}.
In this paper, we tackle the system and data heterogeneity in the FL by considering the \textit{goal-directed} client selection problem based on the FL\&FA and FL\&DA frameworks, which directly establishes a connection between the trained model and the original dataset.
Second, the number of combinations of the selected clients is typically huge, especially when a large-scale network is considered.
This results in prohibitive computational complexity as the client selection problem is typically NP-hard \cite{shi2022vfedcs}.
To overcome these challenges, we design two computation-efficient UCB algorithms for the \textit{goal-directed} client selection problem in the FL\&FA and FL\&DA frameworks, respectively.

\subsection{Contributions}
In this paper, we study the \textit{goal-directed} client selection problem under the FL\&FA and FL\&DA  frameworks by selecting a subset of clients to maximize all clients' average opinions subject to limited communication resources.
The main contributions of this work are listed as follows.
\begin{itemize}
	\item We first propose two closed-loop model analytics frameworks (i.e., the FL\&FA and FL\&DA frameworks)  to evaluate the trained global model.  Unlike existing works that focus on optimizing the training efficiency, we propose to maximize the performance (or the goal) of the well-trained global model.
	\item Based on the proposed frameworks, we tackle the system and data heterogeneity in the FL by considering the \textit{goal-directed} client selection problem. This problem is formulated as the SMAB problem. We propose a  quick initial UCB (Quick-Init UCB) algorithm \cite{chen2016combinatorial} and a BP-UCB algorithm to solve the SMAB problem under the FL\&FA and FL\&DA frameworks, respectively.
	\item We derive two regret upper bounds for the Quick-Init UCB algorithm and the BP-UCB algorithm, respectively, based on the convergence of the BP algorithm and the regret analysis of the UCB algorithm. Theoretical results show that the proposed algorithms are asymptotically optimal when the time horizon is sufficiently large.
	\item  We conduct extensive simulations to evaluate the proposed algorithms with different configurations.
The numerical results demonstrate that the proposed algorithms are close to the optimal solution, with a marginal gap of less than $1.44\%$ and $3.12\%$, respectively. In addition, they enjoy a faster convergence rate than existing client selection algorithms.
\end{itemize}

\subsection{Organization}
The structure of this paper is organized as follows.
Section \ref{sys mod} introduces the system models.
Section \ref{problem_formulation} presents the proposed frameworks and the problem formulation.
Then, the Quick-Init UCB algorithm and its regret upper bound are given in Section \ref{OUS1} under the FL$\&$FA framework;
while the BP-UCB algorithm and its regret upper bound are given in Section \ref{OUS2} under the FL\&DA framework.
Section \ref{Simulation Results} provides extensive numerical results to evaluate the proposed algorithms.
Finally, this paper is concluded in Section \ref{Conclusion}.

\section{System Model}	\label{sys mod}
We consider a wireless communication system that consists of $N$ clients and a base station (or central server).
The $N$ clients are collaborative in learning a global model without the data collection in the central server,
i.e., only the local model parameters are required to upload to the server.
\com{Let $n=1,2,\ldots,N$ be the index of the client and $n \in \mathcal{N}$.}
Assume that time is slotted in $t = 1,2,\ldots,T$, where $T$ is the time horizon.
At each time slot $t$, there are total $L$ communication rounds\footnote{Note that length $L$ can be adjusted flexibly according to the time complexity of the training algorithm and the target accuracy.} for the FL process (i.e., $l=1,\ldots,L$).
In addition, we consider a resource-constrained scenario where only $K$ clients (i.e., $K<N$) can update their parameters to the server at each time slot.
The system's goal is to maximize the clients' average opinions of the trained global model by selecting $K$ clients in the FL process subject to the limited channel budget and the client system and data heterogeneity.

\subsection{Channel Model}
In practice, the communication channel is typically stochastic and time-varying.
Without loss of generality, we consider a block-fading channel model \cite{biglieri2001limiting} and \cite{tong2023model}, where the channel gain is constant during the $l$-th communication round but changes among two blocks (or two communication rounds) \footnote{\com{A more comprehensive study of the influence of the channel variation on FL can be found in \cite{shah2022robust} and \cite{amiri2020federated}.}}.
In addition, we assume that the duration of each communication round $l$ is fixed.
The training time for different clients may vary because of the client system heterogeneity, i.e., different clients have different computation resources.
Thus, clients with poor computation power may fail to update their model parameters to the central server due to the fixed communication duration.

Let $\theta_{n}$  be the successful update probability that client $n$ transmits to the central server through the block-fading channel.
In this work, we consider that $\theta_{n}$ models the following three events: 1) the local training time of client $n$ should be shorter than the duration of the communication round;
2) the trained model parameter $\boldsymbol{\omega}$ should be successfully transmitted to the central server over the block-fading channel;
3) \com{client $n$ does not dropout in this communication round.}
To be specific, we define a binary variable $X^l_{n}$, indicating the transmission outcome at training round $l$, where $X^l_n=1$ means that the $n$-th client successfully updates to the server at round $l$; otherwise, $X_n^l=0$.
Then, we have $\theta_{n} = \mathbb{E}[X_n^l]$, where $\mathbb{E}[\cdot]$ is the expectation operation.

\subsection{Federated Learning}
In the FL, each client keeps its data locally while the local model parameters are uploaded to the central server for model aggregation.
In the following, we introduce the FedAvg algorithm \cite{mcmahan2017communication},  which aggregates the uploaded parameters in the server using the weighted average operation.
Mathematically, the goal of FL is to solve the following stochastic optimization problem
\begin{equation}%\small
	\min _{\boldsymbol{\omega}} {F}\left( \boldsymbol{\omega}  \right) = \frac{1}{N}\sum\limits_{n = 1}^N {{F_n}\left( \boldsymbol{\omega}\right)},
	\label{FL target}
\end{equation}
where $\boldsymbol{\omega}$ is the model parameter vector and $F_n(\cdot)$ is the loss function of client $n$.
Due to the limited communication resources, only $K$ clients will be selected to perform model training.
Denote $\mathcal{K}^t$ by the set of selected clients at time slot $t$.
In the FedAvg algorithm, one needs to compute the weighted average of $K$ selected clients' parameter vectors $\boldsymbol{\omega}^{l}_k, k\in \mathcal{K}^t$.
Here, we consider a modified version of the weighted average operation which captures the influence of the client system heterogeneity, i.e.,
\begin{equation}%\small
{{\boldsymbol{\omega}} ^{l + 1}\left( \mathcal{K}{^t}\right)} = \frac{1}{{ \sum\limits_{k \in \mathcal{K}^t} {X_k^l \left| {{D_k}} \right|} }}\sum\limits_{k \in \mathcal{K}^t} {X_k^l \left| {{D_k}} \right|\boldsymbol{\omega} _{k}^{l }},
	\label{ave01}
\end{equation}
where $\left| {{D_k}} \right|$ is the data size of client $k$ in set $\mathcal{K}^t$.
\com{
Note that the channel variation is modeled in $X^l_{n}$.
From a long-term perspective, the weight in \eqref{ave01} is $\mathbb{E}[X_k^l |D_k|] = \theta_{k} |D_k|$, which captures the channel variation and the computation power heterogeneity.
According to  \cite{abdulrahman2020fedmccs}, the data size will influence the performance of the trained model.
As a result, the channel variation can impact the performance of the trained model by affecting the number of clients participated in the federated training process (i.e., changing the data size).}

The FedAvg algorithm consists of the local training and central aggregation parts.
In the local training part, only $K$ selected clients are required to train the local model $\boldsymbol{\omega}_k$ with their own dataset by using the stochastic gradient descent (SGD) method.
The batch size and epoch size parameters in the SGD method are $B$ and $E$, respectively.
After training, the model parameter vectors $\boldsymbol{\omega}$ are required to upload to the server for model aggregation.
In the central aggregation part, the central server receives all the model parameter vectors from the $K$ selected clients.
Then, it aggregates these parameters using the weight average method in \eqref{ave01}.
After $L$ communication rounds, a well-trained global model $\boldsymbol{\omega} (\mathcal{K}^t)$ can be obtained.
Then, the central server broadcasts the trained model to all $N$ clients with a dedicated channel.
Existing works stop at this step, ignoring the performance (or the goal) of this trained global model.
In this paper, we design two model analysis frameworks (i.e., the FL\&FA and FL\&DA frameworks) to evaluate the trained global model.
\begin{figure*}[!t]
\centering
\subfloat[The proposed FL\&FA framework.]{\includegraphics[width=2.8in]{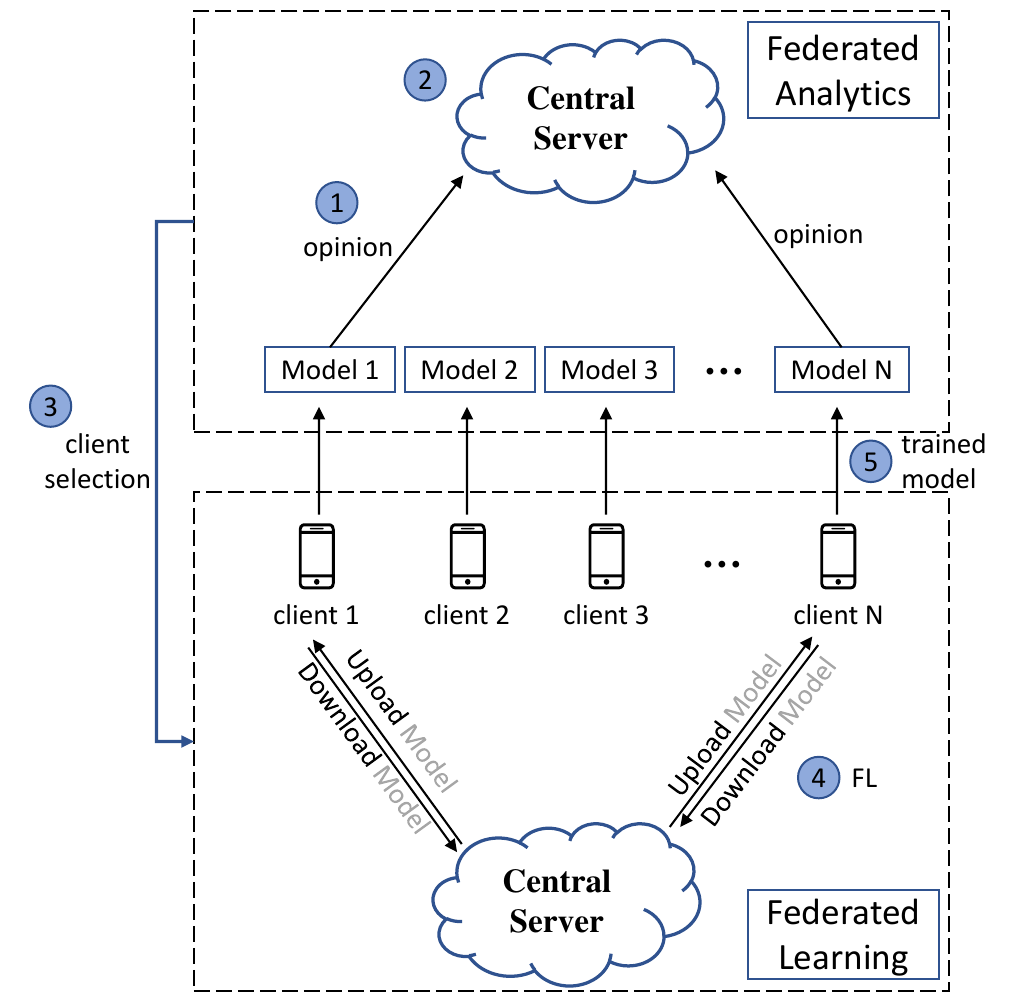} %调节单个子图大小
\label{Lgradual_case1}}
\hfil
\subfloat[The proposed FL\&DA framework.]{\includegraphics[width=2.8in]{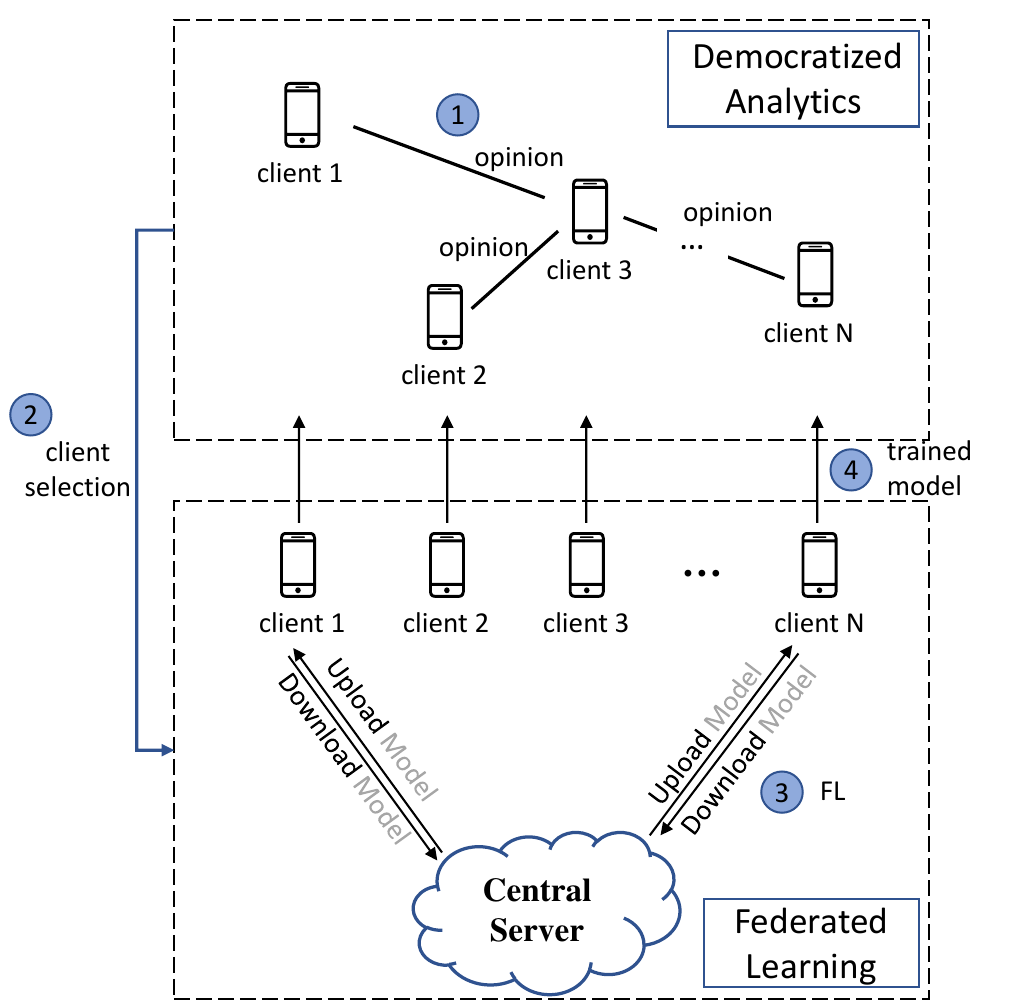}
\label{Lsteep_case2}}
\caption{(a) The structure of the FL\&FA framework; \quad (b) The structure of the FL\&DA framework.}\label{FameW}
\end{figure*}

\subsection{Model Analytics}
The key role of the model analytics part is to evaluate the trained global model after the FL process on the client side.
\com{An evaluation mechanism essentially exists in many practical applications, such as the quality of experience (QoE)-based applications.
Let $r_{t, n}(\boldsymbol{\omega}^t(\mathcal{K}^t))$ be the evaluation score (or opinion) of client $n$ about the trained global model $\boldsymbol{\omega}^t(\mathcal{K}^t)$ at time slot $t$.
Note that the evaluation function $r(\cdot)$ can be an arbitrary mechanism that quantifies the trained global model with the client's local data.}
To preserve privacy, the model analytics framework should follow the FL architecture that uses the local experience\footnote{This local experience can be testing dataset or behaviors at each client, which is different from the training dataset in the FL process.} without the data collection in the central server when evaluating the quality of the trained global model.
For example, the evaluation function can be the calculation function about prediction accuracy that calculates the differences between actual and predicted values.
Therefore, the clients can utilize this evaluation function to measure the prediction accuracy of the trained global model.
In this way, the higher the prediction accuracy, the better performance of the trained global model.
Note that opinion $r_{t,n}(\boldsymbol{\omega}^t(\mathcal{K}^t))$ is a stochastic variable \textit{w.r.t.} time slot $t$ due to the channel variation and the property of the  SGD method.
Without loss of generality, we assume that $r_{t,n}(\boldsymbol{\omega}^t(\mathcal{K}^t))$ changes in the range of $[0,1]$.

There are two differences between the model analytics part and the FL part.
First, the model analytics framework is performed by drawing conclusions from data rather than training a model in the client with high computation requirements.
Second, due to privacy concerns,  clients may prefer to share information with their neighbors rather than upload it to the central server.
Therefore, clients may not want to upload their opinions to the central server.

\section{Frameworks and Problem Formulation} \label{problem_formulation}
We first introduce two closed-loop model analytics frameworks, i.e., the FL\&FA framework in Section \ref{TQIA0} and the FL\&DA framework in Section \ref{BPUCB_FRWRK}.
Based on the proposed frameworks, we present the \textit{goal-directed} client selection problem and formulate it as a SMAB problem in Section \ref{ClinetSelection}.
Finally, we briefly discuss the solution to the SMAB problem under the context of the closed-loop model analytics framework.

\subsection{The FL\&FA Framework} \label{TQIA0}
The structure of the FL\&FA framework is shown in Fig. \ref{FameW}a.
It can be seen that it contains the FA part and the FL part, where the FA part is performed as follows.
At time slot $t$, each client first receives the trained global model ${{\boldsymbol{\omega} ^t}\left( \mathcal{K}{^t}\right)}$ from the central server.
Notice that this model is obtained at the end of the FL process by running the FedAvg algorithm.
Then, all $N$ clients evaluate this global model in the model analysis part.
Thus, the average opinion of all clients in the central server is calculated by
\begin{equation}\label{est_opinion}
r_{t,\cdot} \left( {{\boldsymbol{\omega} ^t}\left( \mathcal{K}{^t}\right)} \right) = \frac{1}{N} \sum_{n = 1}^N {{r_{t,n}}\left( {{\boldsymbol{\omega} ^t}\left( \mathcal{K}{^t}\right)} \right)}.
\end{equation}
Obviously, a larger average opinion $r_{t,\cdot}$ contributes to a better-trained global model.
Based on these testing results,  the central server can guide clients' training parameters in the next round of FL.
Consequently,  a new global model can be obtained after the FL process in time slot $t+1$.
By repeating these steps, model efficacy can be obtained by maximizing the average opinion $r_{T,\cdot} \left( {{\boldsymbol{\omega} ^T}\left( \mathcal{K}{^T}\right)} \right) $ over time horizon $T$.

To sum up, the FL\&FA framework consists of five steps and proceeds in a closed-loop form.
The details of these steps are given below.
\begin{enumerate}
  \item Each client evaluates the trained global model with its local dataset. Then, the generated opinion is uploaded to the central server through a dedicated channel.
  \item The central server receives opinions from all clients. Based on these opinions, some results can be drawn from the original data.
  For example, in the client selection problem,  $K$ out of $N$ clients will be selected to perform model training at the next time slot based on the evaluation results.
  \item The central server sends decision information to the selected clients. Then, these clients are activated to participate in the FL process.
  \item The trained models of the $K$ clients are uploaded to the central server for model aggregation by using the FedAvg algorithm.
  \item After the FL process, the trained global model is broadcast to all clients.
\end{enumerate}

\subsection{The FL\&DA Framework} \label{BPUCB_FRWRK}
The structure of the FL\&DA framework is shown in Fig.~\ref{FameW}b, where the DA part is decentralized.
The key difference between  Figs.~\ref{FameW}a and \ref{FameW}b is that there is no central server in the latter to coordinate the message exchange between clients.
Instead, we assume that the $N$ distributed clients are self-organized, which forms an undirected graph, i.e., $\mathcal{G} = \{\mathcal{V}, \mathcal{E}\}$.
In this graph, clients are viewed as the set of vertices denoted by $\mathcal{V}$,
and $\mathcal{E}$ is the set of edges,  which can be further written as an $N$-by-$N$ matrix.
That is, $\mathcal{E} = [\boldsymbol{e}_1; \boldsymbol{e}_2; \ldots; \boldsymbol{e}_N]$,
where the $n$-th row vector is $\boldsymbol{e}_n = [e_{n,1}, e_{n,2}, \ldots, e_{n,N} ]$.
If $e_{i,j} = 1$, this means that client $i$ links to client $j$, and $e_{i,j} = 0$ otherwise.
We further assume the relationship is symmetric, i.e., $e_{i,j} = e_{j, i}$.
For example, in Fig.~\ref{BP}, client $i$ links with neighbors 1, 2, and 3.
Thus,  client $i$ can exchange messages with not only its neighbors but also those connected to its neighbors, such as clients 4 and 5.
In this way, each client's information that should be selected is a combination of other clients'  opinions.
This inference is realized in step 1 of  Fig.~\ref{FameW}a, where clients exchange the information using the BP algorithm.

Let $\mathbb{I}^{t} _n$ be the belief of client $n$, where $\mathbb{I}^{t} _n=1$ means that client $n$ should be selected in time slot $t$; $\mathbb{I}^{t} _n=0$, otherwise.
In fact, the better the evaluation score of the trained global model, the higher the probability that a client should be selected.
Define ${b_{t,n}}\left( {{\mathbb{I}^{t}_n}} \right) $  as the belief of client $n$ at time slot $t$.
According to \cite{yedidia2003understanding}, the calculation of ${b_{t,n}}\left( {{\mathbb{I}^{t}_n}} \right) $ is given by
\begin{equation}
	{b_{t,n}}\left( {{\mathbb{I}^{t}_n}} \right) = {\kappa_n}{\phi _n}\left( {{\mathbb{I}^{t}_n},{r_{t,n}}\left( {{\boldsymbol{\omega} ^{t}}} \right)} \right)\!\prod\limits_{i=1}^{N} e_{i,n}{{m_{i,n}}\left( {{\mathbb{I}^{t}_n}} \right)},
	\label{Belief}
\end{equation}
where $\kappa_n$ is the normalization factor and $\phi_n$ is the local function of client $n$.
Here, we define the  local function as the power of the client's opinion, i.e.,
\begin{equation}
	{\phi _n}\left( {{\mathbb{I}^{t}_n},{r_{t,n}}\left( {{\boldsymbol{\omega} ^t}} \right)} \right) = {\left( {{r_{t,n}}\left( {{\boldsymbol{\omega} ^t}} \right)} \right)^2}.
	\label{local}
\end{equation}
In addition, term $m_{i,n}(\mathbb{I}^{t}_n)$ in \eqref{Belief} denotes the message passing function that client $i$ transmits to client $n$ under state $\mathbb{I}^{t}_n$.
Mathematically, $m_{i,n}$ can be obtained by iteratively updating the following equation
\begin{equation}\label{message}
	\begin{aligned}
		& \!{m_{i,n}}\left( {{\mathbb{I}^{t}_n}} \right)=	\\
		& \!\sum\limits_{{\mathbb{I}^{t}_i} \in \left\{ {0,1} \right\}}\! {{\phi _i}\!\left( {{\mathbb{I}^{t}_i},{r_{t,i}}\left( {{\boldsymbol{\omega} ^t}} \right)} \right){\psi _{i,n}}\!\left( {{\mathbb{I}^{t}_i},{\mathbb{I}^{t}_n}} \right)}\! \prod\limits^{N}_{k =1, k\neq n} e_{k,i}{{m_{k,i}}\left( {{\mathbb{I}^{t}_i}} \right)},
	\end{aligned}
\end{equation}
where
\begin{equation}
	{\psi _{i,n}}\left( {{\mathbb{I}^{t}_i},{\mathbb{I}^{t}_n}} \right) = \exp \left( { - Cd_{{i},{n}}^\beta } \right)
	\label{compa}
\end{equation}
is the compatibility function between clients $n$ and $i$.
Without loss of generality, we adopt the large-scale fading in \cite{yuan2012defeating} as the compatibility function of \eqref{compa},
where $C$ and $\beta$ are two constants of the fading-relative parameters.
In addition, ${d_{i,n}}$ denotes the Euclidean distance between clients $i$ and $n$, which characterizes the correlation of two clients' data.
Therefore, a larger distance leads to a lower value of the compatibility function.
\begin{figure}[t]
	\centering
	\includegraphics [scale=0.45,trim=0 0 0 0]{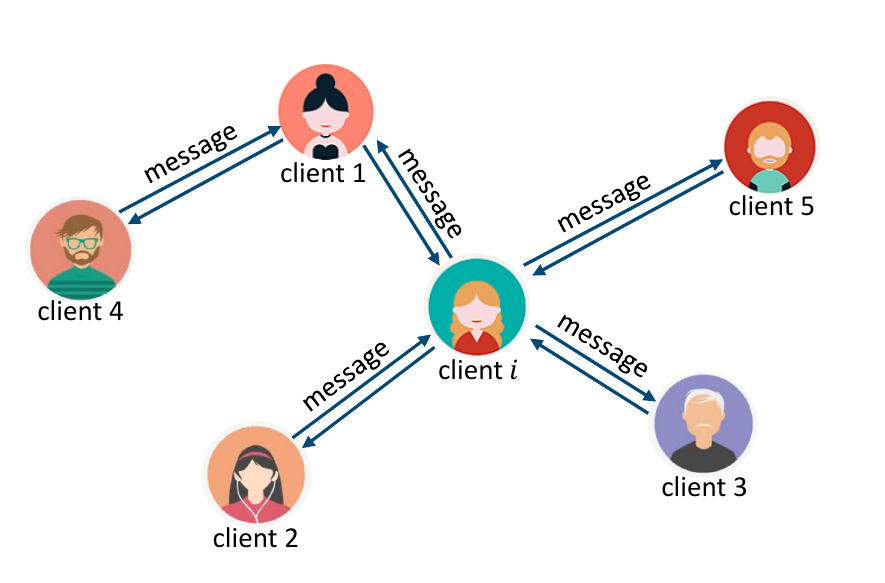}
	\caption{Message exchange among client $i$ and its neighbors $1$, $2$, and $3$.}
	\label{BP}
\end{figure}

To sum up, the FL\&DA framework consists of five steps and proceeds in a close loop.
The details of these steps are given below.
\begin{enumerate}
  \item Each client receives the global model from the central server. Then, it evaluates the trained global model with its local dataset.
  \item Clients exchange their beliefs ${b_{t,n}}$  through the decentralized network with the BP algorithm.
  \item Based on these beliefs and the decentralized network, each client ranks the opinion $r_{t,n}$ using the gossip method.  Then,  $K$ out of $N$ clients are activated to participate in the FL process.
  \item The trained models of the $K$ clients are uploaded to the central server for model aggregation by using the FedAvg algorithm.
  \item After the FL process, the trained global model is broadcast to all clients.
\end{enumerate}

\subsection{Goal-Directed Client Selection Problem}\label{ClinetSelection}
Based on the FL\&FA and FL\&DA frameworks, we further investigate the client selection problem in FL.
The flow diagram of the client selection framework is shown in Fig.~\ref{FA}.
At time slot $t$, each client first receives the trained global model ${{\boldsymbol{\omega} ^t}\left( \mathcal{K}{^t}\right)}$ from the central server.
Notice that this model is obtained at the end of the FL process by running the FedAvg algorithm.
Then, all $N$ clients evaluate this global model using their testing dataset in the model analysis part.
Thus, the average opinion of all clients is given by
\begin{equation}\label{est_opinion}
r_{t,\cdot} \left( {{\boldsymbol{\omega} ^t}\left( \mathcal{K}{^t}\right)} \right) = \frac{1}{N} \sum_{n = 1}^N {{r_{t,n}}\left( {{\boldsymbol{\omega} ^t}\left( \mathcal{K}{^t}\right)} \right)}.
\end{equation}
Obviously, a larger average opinion $r_{t,\cdot}$ contributes to a better-trained global model.
Based on all clients' opinions,  $K$ out of $N$ clients will be selected to train a global model by using a specific selection strategy.
Consequently,  a new global model can be obtained after the FL process in time slot $t+1$.
By repeating these steps, an optimal set of clients can be determined to maximize the average opinion $r_{T,\cdot} \left( {{\boldsymbol{\omega} ^T}\left( \mathcal{K}{^T}\right)} \right) $.
\begin{figure}[!t]
	\centering
	\includegraphics [scale=0.60,trim=0 0 0 0]{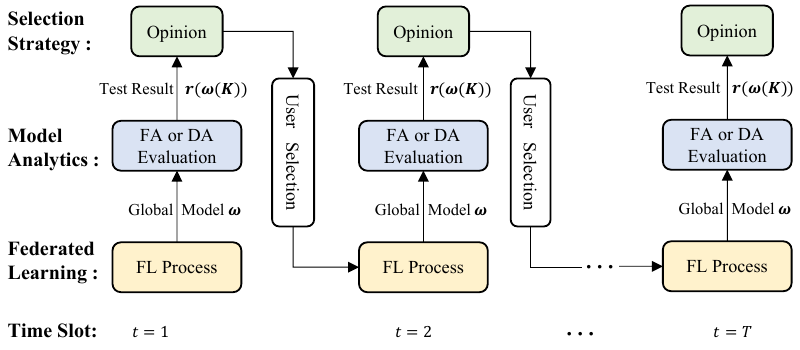}
	\caption{An illustration of the client selection framework.}
	\label{FA}
\end{figure}

We aim to improve the performance (i.e., the average opinion) of the trained global model by selecting $K$ desirable clients over time horizon $T$ subject to limited communication resources.
This problem can be formulated as an online stochastic optimization problem, i.e.,
\begin{equation}\label{target}
	\begin{aligned}
		& \underset{}{\max\limits_{\mathcal{K}}}
		& & \frac{1}{T}\frac{1}{N}\sum\limits_{t = 1}^T  {\sum\limits_{n = 1}^N \mathbb{E}\left[ {{r_{t,n}}\left( {{\boldsymbol{\omega} ^t}\left( \mathcal{K}{^t}\right)} \right)} \right]  } \\
		& \text{s.t.}
		& & \mathcal{K}{^t} \subseteq  \mathcal{N}, t = 1, \ldots, T,\\
		& & & \left| \mathcal{K}{^t} \right| = K, t = 1, \ldots, T,
	\end{aligned}
\end{equation}
where  the second constraint is the cardinal of set $\mathcal{K}{^t}$, i.e., the communication budget is $K$.
Intuitively, it can be solved directly using the exhaustive search method when the feasible solutions are discrete and finite.
However, it is impossible because the evaluation function $r(\cdot)$ is typically stochastic and unknown, which may become better understood as time is involved.
Hence,  we first need to estimate the mean value of $r_{t,n} \left( {{\boldsymbol{\omega} ^t}\left( \mathcal{K}{^t}\right)} \right)$.
However, due to the many combinations, the estimate-then-commit method will take a long time in the model training phase.
For example, when considering a system with $N=20$ clients and  $K=5$ communication budgets,  the calculation of $5$ out of $20$ is over $10^4$.
One way to overcome this drawback is to estimate the mean value of $r_{t,n} \left( {{\boldsymbol{\omega} ^t}\left( \mathcal{K}{^t}\right)} \right)$ while adaptively choosing a subset of clients for the model training.
This incurs the EE dilemma where one needs to exploit the client set with the maximum estimated mean reward as many times as possible to boost its accumulated rewards. It is also necessary to explore the other client sets as often as possible to avoid missing the optimal one.
This motivates us to employ the SMAB framework to model  problem~\eqref{target} to better balance the EE dilemma.

\section{Quick-Init UCB Based Client Selection in the FL\&FA Framework} \label{OUS1}
We first give some definitions of the SMAB problem under the  FL\&FA framework.
Then, we present the  UCB algorithm to solve the stochastic bandit problem.
To reduce the large arm space, we propose the Quick-Init UCB algorithm in Section \ref{TQIA}.
Finally, we derive a regret upper bound for the Quick-Init UCB algorithm in Section \ref{PA1}.

\subsection{The Quick-Init UCB Algorithm}\label{TQIA}
In this SMAB problem, the player in the FL\&FA framework is the central server, and the arms are the combinations of the $K$ clients, i.e., the $K$ out of $N$ clients.
Let $\mathcal{A}$ be the set of arms, i.e., $\mathcal{A}=\{a=1,2,\ldots,|\mathcal{A}|\}$.
At each time slot $t$, the player first picks up an arm from the arm set $\mathcal{A}$.
Then, the selected $K$ clients are collaborative to train a global model ${{\boldsymbol{\omega} ^t}\left( \mathcal{K}{^t}\right)}$ in the FL process.
After training, all $N$ clients will receive the trained global model and test it with their local dataset in the FA phase.
By collecting all clients' opinions to the central server, the player will obtain a corresponding reward, i.e., ${r_{t, a}}\left( {{\boldsymbol{\omega} ^t}\left( \mathcal{K}{^t}\right)} \right), \forall a \in \mathcal{A}$.
With some conceptual ambiguity, let ${r_{t,a}}\left( {{\boldsymbol{\omega} ^t}\left( \mathcal{K}{^t}\right)} \right)= {r_{t,\cdot}}\left( {{\boldsymbol{\omega} ^t}\left( \mathcal{K}{^t}\right)} \right)$, where the latter is given in \eqref{est_opinion}.
Let $\mathbb{I}^{t} _a$ be an indicator that $\mathbb{I}^{t} _a =1$ denotes arm $a$ is selected at time slot ${t}$; otherwise $\mathbb{I}^{t} _a = 0$.
Then, the average empirical reward of arm $a$ is defined as
\begin{equation}
	\bar{r}_{t,a}  = \frac{1}{{{\Pi_{t,a}}}}\sum\limits_{i = 1}^t \left( {\mathbb{I}^{i} _a} \times {{r_{i,a}}\left( {{\boldsymbol{\omega} ^{i}}\left( \mathcal{K}{^{i}}\right)} \right)} \right),
	\label{exploit}
\end{equation}
where $\Pi_{t,a}$ is the number of times that arm $a$ has been selected up to time slot $t$.

We employ the UCB algorithm to trade off the exploitation and exploration dilemma raised in the above stochastic bandit problem.
According to \cite{auer2002finite}, the UCB index is constructed by
\begin{equation}\label{ucb-index}
\Psi^{\mathrm{ucb}}_{t,a} = \bar{r}_{t,a}  + \mu \sqrt {\frac{{\ln  t }}{{{\Pi_{t,a}}}}}, \ \forall a \in \mathcal{A},
\end{equation}
where $\mu$ is a hyperparameter that controls the exploration degree.
Note that the first term of \eqref{ucb-index} is used for exploitation;
while the second term is an upper bound of $\bar{r}_{t,i}$ that accounts for exploration.
The UCB algorithm selects an arm with the largest UCB index at each time slot.
However, this UCB algorithm cannot be directly applied to our stochastic bandit problem since the number of the arm can be very large with the combinational operation.
It may take a long time for the algorithm to traverse all arms.

To overcome this challenge, we propose a Quick-Init UCB algorithm to reduce the algorithm's initiation phase.
The Quick-Init UCB algorithm only explores each client once a time when initiating the empirical reward $\bar{r}_{t,a}, \forall a \in \mathcal{A}$ at cold-start time.
In this way, the number of initialization time slots in the UCB algorithm is decreased from  $|\mathcal{A}|$ to $N$ in the Quick-Init UCB algorithm.
Specifically, it first divides the $N$ clients into $N / K$ groups called cold-start arms, and each group contains $K$ clients.
Each cold-start arm is sequentially selected at the first $t_0 = N/K$ time slots.
After testing, the server obtains the reward for all clients' average opinions.
Note that each reward ${r_{t, a}}\left( {{\boldsymbol{\omega} ^t}\left( \mathcal{K}{^t}\right)} \right)$ is correlated to the selected client set $\mathcal{K}_t$.
Then, the Quick-Init UCB algorithm adds the reward of client $i$ in set $\mathcal{K}_t$ to the arms containing client $i$.
As a result, all arms can be initialized within $t_0 = N/K$ time slots.

The pseudo-code of the Quick-Init UCB algorithm is shown in Algorithm \ref{cen_ALG}.
In the initialization phase, the central server chooses each client once according to the random combination of the $K$ clients.
Thus, the total time slots at the initialization phase are ${N \mathord{\left/ {\vphantom {N K}} \right. \kern-\nulldelimiterspace} K}$.
At each time slot $t$, the central server first computes the UCB index $\Psi^{\mathrm{ucb}}_{t,i}$ of each arm.
Then, it selects an arm $\mathbb{I}_a^t$ (i.e., the client set $\mathcal{K}^t$) with the highest UCB index.
By activating the $K$ clients in the set $\mathcal{K}^t$, a trained global model can be obtained by running the FedAvg algorithm.
In the FA phase, each client downloads and evaluates the trained global model using their local dataset.
After that, each client's test opinion should be uploaded to the central server.
The reward of arm $\mathbb{I}_a^t$ is the average opinion of all clients, and then the server updates the UCB index based on this reward.
Algorithm \ref{cen_ALG} repeats the above steps until the stopping time $T$ is reached.
\begin{algorithm}[!t]
	\caption{Quick-Init UCB algorithm in the FL\&FA framework}
	\label{cen_ALG}
	\begin{algorithmic}[1]
		\Require {client set $\mathcal{N}$, arm set $\mathcal{A}$, data size $D$, and communication budget $K$}%and all clients' opinions
		\Ensure {the selected client set $\mathcal{K}^{t-1}$ at time slot $t$}
		\State {Initialization: Divide clients into $N/K$ groups as the cold-start arms.
		Choose each cold-start arm at the first $t_0 = N/K$ time slots to approximate the initial empirical reward $\bar{r}_{t,a}$; set the selected time ${\Pi_{t,a}} = 1, a \in \mathcal{A}$ and the hyper-parameter $\mu$.}
		\For {$t = {t_0} + 1, \ldots, T$, the central server}
		\For { $a = 1, \ldots, |\mathcal{A}|$}
		\State {compute the UCB index $\Psi^{\mathrm{ucb}}_{t,a}$}  using~\eqref{ucb-index}
		\EndFor
		\State Select arm $\mathbb{I}_a^t$ (or set $\mathcal{K}^t$) with the highest UCB index
        \State The $\mathcal{K}^t$ clients train the model using SGD method
        \State Obtain the global model $\boldsymbol{\omega}^{t} \left(\mathcal{K}^t\right) $ using \eqref{FL target} and \eqref{ave01}
       \State \textbf{$\blacktriangleright$ Federated Analytics:}
       \For { $n = 1, \ldots,N$, each client}
		\State download the global model $\boldsymbol{\omega}^{t} \left(\mathcal{K}^t\right) $ from server
       \State  tests ${\boldsymbol{\omega} ^t}\left( \mathcal{K}{^t}\right)$ with its local data
       \State  upload the opinion ${{r_n}\left( {{\boldsymbol{\omega} ^t}\left( \mathcal{K}{^t}\right)} \right)}$ to the server
		\EndFor
         \State Central server calculates the reward $\bar{r}_{\mathbb{I}^t_a}$ using \eqref{est_opinion}
		\State {Update the selected frequency ${\Pi_{t,{\mathbb{I}^t_a}}} = {\Pi_{t-1,{\mathbb{I}^t_a}}} + 1$}
		\State Update the empirical average opinion $\bar{r}_{t,{\mathbb{I}^t_a}} $
		\EndFor		
	\end{algorithmic}
\end{algorithm}

\subsection{\com{Complexity Analysis and Implementation Consideration}}\label{SecICCA}
\com{
\textbf{Complexity:} We first give a complexity analysis for the Quick-Init UCB algorithm.
In Algorithm \ref{cen_ALG}, the input parameters are the client set $\mathcal{N}$, arm set $\mathcal{A}$, data set $D$, and communication budget $K$.
At each time slot, the UCB algorithm in lines 5 and 6 has the computational complexity of $\mathcal{O}(|\mathcal{A}| + |\mathcal{A}|^2)$,
where the second term comes from the sorting operation.
The FL training processes of the $K$selected clients contribute to the complexity of $\mathcal{O}(KBE)$, where $B$ and $E$ are the batch size and epoch size
parameters in the SGD method, respectively.
The evaluation step at line 3 has a $\mathcal{O}(ND)$ complexity.
Thus, the total computational complexity of Algorithm \ref{cen_ALG} is  $\mathcal{O} \left(  T|\mathcal{A}| + T|\mathcal{A}|^2 + TK B E +N DT \right)$.
If the client has the same training and testing data size,  the total computational complexity of Algorithm \ref{cen_ALG} is about  $\mathcal{O}\left(  T|\mathcal{A}|^2 +N DT \right)$, where $|\mathcal{A}| = \binom{N}{K}$.
Therefore, we see that the computational complexity mainly depends on the system parameters of $N$, $K$, and $D$.}

\com{\textbf{Implementation:}
We then discuss the practical implementation challenges of the Quick-Init UCB algorithm in real-world scenarios from various aspects, including technical considerations, privacy concerns, security implications, and scalability issues.
The Quick-Init UCB algorithm is structured based on the FL\&FA framework, as referenced in Section \ref{TQIA0}. Within this centralized framework, we assume that communication between clients and the central server is synchronized. In this setup, clients should upload their model parameters and test opinions to the central server while their local data remains cached on the client side.
To account for system heterogeneity, data heterogeneity, and channel variation, we incorporate these factors into the model aggregation algorithm, represented by Eqs. \eqref{FL target} and \eqref{ave01}. This approach allows the proposed algorithm to accommodate diverse practical implementations.

However, the Quick-Init UCB algorithm encounters scalability challenges, particularly in high-density networks. The issue arises due to the exponential growth in feasible solutions when selecting $K$ clients. For instance, with $N=100$ and $K=5$, the exploration space becomes $|\mathcal{A}| =\binom{100}{5}>10^7$, surpassing the memory capacity of conventional computers.
We propose the FL\&DA framework in Section \ref{OUS2} to tackle this scalability concern. The FL\&DA framework offers a potential solution to mitigate the scalability problem using a decentralized network topology.}

\subsection{Performance Analysis}\label{PA1}
At last, we derive a regret upper bound for the Quick-Init UCB algorithm.
First, we adopt the \emph{pseudo-regret} as the performance metric for this stochastic bandit problem,
defined as the expected cumulative performance loss between the optimal and currently adopted policies.
Mathematically, it can be defined as
\begin{equation}\small
	\begin{aligned} \label{reg}
{\mathop{\rm Reg}\nolimits}\left( T \right) &= \mathbb{E}\left[
{\sum\limits_{t = 1}^{T} {{r_{t, a^{\ast}}}\left( {{\boldsymbol{\omega} }\left( {{\mathcal{K}^*}} \right)} \right)} } \!-\! \sum\limits_{t = 1}^T{\sum\limits_{a = 1}^{|\mathcal{A}|} \mathbb{I}^{t} _a {{r_{t,a}}\left( {{\boldsymbol{\omega} ^t}\left( {{\mathcal{K}^t}} \right)} \right)} } \right]\\
& = \sum\limits_{t = 1}^{T} \hat{r}\left(\mathcal{K}^{\ast} \right) - \sum_{t = 1}^T\sum_{a = 1}^{|\mathcal{A}|}  \hat{r}\left(\mathcal{K}_a \right)  \mathbb{E}\left[ \mathbb{I}^{t} _a \right] \\
& = \sum_{a=1}^{\left|\mathcal{A} \right|} \Delta_a \mathbb{E}\left[\Pi_{T,a}\right],
	\end{aligned}
\end{equation}
where the expected reward of arm $a$ (i.e., client set $\mathcal{K}_a$) is denoted as $\hat{r}(\mathcal{K}_a) = \mathbb{E}\left[{ {{r_a}\left( {{\boldsymbol{\omega}}\left( {{\mathcal{K}_a}} \right)} \right)} } \right]$, the gap $\Delta_a$ is $\mathbb{E}\left[r_{a^{\ast}}\left(\boldsymbol{\omega} \left(\mathcal{K}^{\ast}\right) \right)\right] -\hat{r}(\mathcal{K}_a)$, and
$\Pi_{T,a}$ is the number of times that arm $a$ has been selected up to time slot $T$.
The optimal client set $\mathcal{K}^{\ast}$ (i.e., arm $a^{\ast}$) is obtained by solving problem \eqref{target}.
Note that the first expectation is operated on the random reward and the selection strategy;
while the second and third expectations are operated on the random selection strategy.

Then, we define the set of bad arms by
\begin{equation}\label{bad_set1}
\mathcal{K}_{\mathrm{B}}=\left\{\mathcal{K}_a \mid \hat{r}(\mathcal{K}_a)
< \operatorname{opt}_{\boldsymbol{r}}\right\},
\end{equation}
where
$\mathcal{K}_a$ is the client set in arm $a$ and $\operatorname{opt}_{\boldsymbol{r}}=\mathbb{E}\left[r_{a^{\ast}}\left(\boldsymbol{\omega} \left(\mathcal{K}^{\ast}\right) \right)\right] $ is the optimal expected reward.
Thereafter, we obtain the worst and best reward gaps between the best and worst arm compared with the optimal super arms, respectively, i.e.,
\begin{equation}\label{best0}
R_{\min }=\operatorname{opt}_{\boldsymbol{r}}-\max \left\{\hat{r}(\mathcal{K}_a) \mid \mathcal{K}_a \in \mathcal{K}_{\mathrm{B}}\right\},
\end{equation}
\begin{equation} \label{worst0}
R_{\max}=\operatorname{opt}_{\boldsymbol{r}}-\min\left\{\hat{r}(\mathcal{K}_a) \mid \mathcal{K}_a \in \mathcal{K}_{\mathrm{B}}\right\}.
\end{equation}
Finally, the regret upper bound is given in Theorem~\ref{Theorem00}.
\begin{theorem}\label{Theorem00}
	\textit{For the Quick-Init UCB algorithm with $N$ clients and ${\left|\mathcal{A} \right|}$ arms, assume that the communication budget is $K$.
		The upper  pseudo-regret bound of~\eqref{reg}  is given by
		\begin{equation}
			\begin{split}
				{\mathop{\rm Reg}\nolimits}\left( T \right) \le \frac{{8{\left|\mathcal{A} \right|}\ln \left( T \right)}}{{{R_{\min }}}} + \left( {\frac{{{\pi ^2}}}{3} + 1} \right){\left|\mathcal{A} \right|}R_{\max }.
			\end{split}
	\end{equation}}
\end{theorem}
\begin{proof}
Please see Appendix A.
$\hfill\blacksquare$
\end{proof}

\begin{remark}
Theorem~\ref{Theorem00} shows that the Quick-Init UCB algorithm achieves asymptotically regret bound under this centralized framework.
In other words,  the cumulative regret ${\mathop{\rm Reg}\nolimits}\left( T \right)/T\rightarrow 0$ when $T\rightarrow \infty$,
which indicates that the proposed algorithm converges when the time horizon $T$ is sufficiently large.
\end{remark}

\section{BP-UCB Based client Selection in the FL\&DA Framework} \label{OUS2}
In the previous section, we have considered the centralized model analytics framework, where all clients are required to upload their opinions to the central server.
In some applications, clients may prefer to share information with their neighbors rather than upload it to the central server due to privacy concerns.
This requires a decentralized model analytics framework.
To capture this decentralized feature, we employ the BP algorithm for messages passing among clients.
 Then, we propose a BP-UCB algorithm in Section \ref{BPUCB_ALG} for the client selection problem (or the SMAB problem).
Finally, we provide a convergence analysis for the BP algorithm and a regret upper bound for the BP-UCB algorithm in Section \ref{BPUCB_PA}.

\subsection{The BP-UCB Algorithm}\label{BPUCB_ALG}
In this SMAB problem, the arm is the client.
The reward is the client's opinion $r_{t,n}\left( \boldsymbol{\omega}(\mathcal{K}^t) \right)$.
Note that this differs from the SMAB problem in the FL\&FA framework, where the arm is the combination of the $K$ clients, and the reward is all clients' average opinions.
Thus, the system's goal is to maximize the accumulated rewards by selecting $K$ optimal clients for model training.
However, this will result in an unfair client selection strategy, deviating from our objective in Section \ref{problem_formulation}.
To address this issue, we incorporate the client's belief into constructing the UCB index, which forms a stable probability distribution over all clients.
Formally, the index function of the BP-UCB algorithm for client $n$ at time slot $t$ can be constructed by
\begin{equation}\label{ucb1}%\small
	\Psi^{\mathrm{bpucb}}_{t,n} = \bar{r}_{t,n}  + \frac{1}{N}{b_{t,n}}\left( {{\mathbb{I}^{t} _n}} \right) + \mu \sqrt {\frac{{\ln  t }}{{{\Pi_{t,n}}}}},
\end{equation}
where
\begin{equation}
	\bar{r}_{t,n}  = \frac{1}{{{\Pi_{t,n}}}}\sum\limits_{i = 1}^t \left( {\mathbb{I}^{i} _n} \times {{r_{t,n}}\left( {{\boldsymbol{\omega} ^{i}}\left( \mathcal{K}{^{i}}\right)} \right)} \right).
	\label{exploit02}
\end{equation}
It can be seen that equation \eqref{ucb1} consists of three terms.
The first two terms are related to the estimated mean reward $\bar{r}_{t,n}$, which is used for exploitation;
while the third term is an optimism confidence bound of the estimated mean reward, accounting for exploration.
Thus, it can efficiently balance the EE dilemma.
There are three differences between the Quick-Init UCB algorithm and the BP-UCB algorithm.
First, there is no central server to collect all clients (or arms)' opinions to calculate the UCB index in the BP-UCB algorithm.
Second, there are correlations among $N$ clients in a decentralized FL\&DA framework, i.e., the client will be influenced by its neighbors,
while the FL\&FA framework fails to model this relationship.
Third, the number of selected arms at each time slot is changed to $K$ in the  BP-UCB algorithm instead of selecting only one arm with the highest UCB index in the Quick-Init UCB algorithm.
Although these differences exist, we are interested in studying whether the BP-UCB algorithm can perform better than the Quick-Init UCB algorithm.

\begin{algorithm}[t]
	\caption{BP-UCB algorithm in the FL\&DA framework for client $n$}
	\label{OUS_ALG}
	\begin{algorithmic}[1]
		\Require {clients $\mathcal{N}$, data set $D$, and communication budget $K$.}
		\Ensure {the state of client $n$.}
		\State {Initialization: the average empirical reward $\bar{r}_{1,n}$, the selected times ${\Pi_{1,n}} = 1$, and the hyperparameter $\mu$.}
		\For {$t = 2, \ldots, T$}
		\State {Compute the BP-UCB index $\Psi^{\mathrm{bpucb}}_{t,n}$} using~\eqref{ucb1}
        \State  Exchange these indices using the gossip method
		\If  {$\Psi^{\mathrm{bpucb}}_{t,n}$ is the $K$ highest UCB indices}
         \State Client $n$ trains the model using the SGD method
         \State {Update the selected frequency ${\Pi_{t,n}} = {\Pi_{t-1,n}} + 1$}
         \State {Update the average opinion $\bar{r}_{t,n}$} using \eqref{exploit02}
         \Else {}
         \State {Keep the selected frequency ${\Pi_{t,n}} = {\Pi_{t-1,n}} $}
         \State {Keep the empirical average opinion $\bar{r}_{t,n} = {\bar{r}_{t-1,n}} $}
      \EndIf
      \State \textbf{$\blacktriangleright$ Democratized Analytics:}
      \State {Receive the global model from the server}
         \State {Test the global model ${{r_{t,n}}\left( {{\boldsymbol{\omega} ^t}\left( \mathcal{K}{^t}\right)} \right)}$}
		\State {Compute the local function $\phi_n$ using~\eqref{local}}
		\State {Compute the compatibility function $\psi _{i,n}$ with~\eqref{compa}}
		\State {Compute the message $m_{i,n}$ using~\eqref{message}}
		\State {Compute the belief ${b_{t,n}}$ using~\eqref{Belief}}
		\EndFor		
	\end{algorithmic}
\end{algorithm}
Now, we are in the position to present the BP-UCB algorithm.
The pseudo-code is shown in Algorithm \ref{OUS_ALG}, which runs by each client.
At time slot $t$, each client first computes the BP-related parameters using \eqref{Belief}-\eqref{compa}.
Specifically, each client calculates the local functions as well as the compatibility functions, and then these messages exchange among the $N$ clients until they converge\footnote{Note that the convergence time of the message passing is ignored here because the BP and FA process can be done in different time-scale.}.
Thereafter, each client calculates its belief ${b_{t,i}}\left( {{\mathbb{I}^{t}_i}} \right)$  by \eqref{Belief}.
Based on the belief and test opinion, client $n$ can obtain the BP-UCB index  $\Psi^{\mathrm{bpucb}}_{t,n}$ by \eqref{ucb1}.

The remaining problem is how to determine the order of the $N$ clients' BP-UCB indices.
Here, we adopt the gossip method \cite{foster2004research} for all clients to vote for this rank.
Specifically, in graph $\mathcal{G}$, clients communicate with their neighbors via the links in $\mathcal{E}$ at each time slot.
After a few iterations, all clients can agree on a ranked BP-UCB index.
For example, as shown in Fig.~\ref{BP}, client $i$ sends its BP-UCB index to its neighbors $\{1, 2, 3, 5\}$,  which cache client $i$'s  BP-UCB index and forward it to client $4$.
In this way, all clients know client $i$'s  BP-UCB index and sort a rank.
Repeating the above steps $N$ times, all clients can agree on a ranking of the BP-UCB indices.
Therefore, client $i$  can determine whether it is in the top $K$ according to this order.

Finally, clients with the $K$ highest BP-UCB indices are activated to participate in the FL process by running the FedAvg algorithm.
Note that the FL process is centralized, while the DA is decentralized.
After the FL process, the central server will broadcast the trained global model to all clients.
Thus,  each client can test this global model with its local dataset.
Based on the test opinion and selection strategy,  each client updates its selected frequency and empirical average opinions.
This information will be used to construct the BP-UCB index $\Psi^{\mathrm{bpucb}}_{t,n}$  at the next time slot.

\subsection{\com{Complexity Analysis and Implementation Consideration}}\label{SecICCA02}
\com{
\textbf{Complexity:} We first give a complexity analysis for the BP-UCB algorithm.
In Algorithm \ref{OUS_ALG}, the input parameters are the client set $\mathcal{N}$, data set $D$, and communication budget $K$.
At each time slot, the UCB algorithm in lines 3 and 5 has the computational complexity of $\mathcal{O}(N + N^3)$,
where the second term comes from the sorting operation among $N$ clients.
The gossip method has a complexity of $\mathcal{O}(N^2)$ \cite{georgiou2008complexity}.
The FL training at the selected $K$ clients contributes to the complexity of $\mathcal{O}(KBE)$, where $B$ and $E$ are the batch size and epoch size
parameters in the SGD method, respectively.
The evaluation step at line 15 has a complexity of $\mathcal{O}(ND)$, and the BP algorithm at lines 16-20 has a complexity of $\mathcal{O}(N)$.
Thus, the total computational complexity of Algorithm \ref{cen_ALG} is  $\mathcal{O} \left(  TN^3 + T N^2+N DT +2 N T  + TK B E \right)$.
If the client has the same training and testing data size,  the total computational complexity of Algorithm \ref{cen_ALG} is about  $\mathcal{O}\left(  N^3T +N DT \right)$.
Therefore, the computational complexity mainly depends on $N$ and $D$.}

\com{\textbf{Implementation:}
We then present the practical implementation challenges of the BP-UCB algorithm within real-world scenarios.
The BP-UCB algorithm is structured based on the FL\&DA framework, as referenced in Section \ref{BPUCB_FRWRK}.
In this centralized training and decentralized analytics framework, we assume that communication between clients and the central server is synchronized. In this setup, clients are only required to upload their model parameters to the central server and share their test opinions using the network topology.
The privacy concern is addressed by keeping their local data on the client side.
To account for system heterogeneity, data heterogeneity, and channel variation, we incorporate these factors into the model aggregation algorithm, represented by Eqs. \eqref{FL target} and \eqref{ave01}. This approach allows the proposed algorithm to accommodate diverse practical implementations.
The decentralized analytics framework can overcome the scalability problem, but it also suffers from high communication overhead
due to the BP algorithm's message-passing operation.}

\subsection{Performance Analysis}\label{BPUCB_PA}
We first give a convergence analysis for the BP algorithm.
In the DA framework, the messages need to be exchanged through network graph $\mathcal{G}$ by using the BP algorithm.
In the following, we investigate the conditions of BP convergence under local function \eqref{local} and compatibility function \eqref{compa}.
We denote $d_{\max}$ as the maximum distance between two arbitrary clients in the network.
Consequently, we show the convergence of this iteration in Theorem~\ref{Theorem03}.
\begin{theorem}\label{Theorem03}
\textit{
	If $\left( {N - 1} \right)\tanh \left| {d_{\max }^\beta}\right| < 1$ strictly holds in the network, the BP algorithm will converge to a unique point, irrespective of the initial messages.
}
\end{theorem}
\begin{proof}
Please see Appendix B.
$\hfill\blacksquare$
\end{proof}

\com{Based on \textbf{Theorem}~\ref{Theorem03}, we next give a finite-time analysis of the regret upper bound for the BP-UCB algorithm.
In the BP-UCB algorithm, we can ignore the convergence time of the message passing in the BP algorithm because the duration of the FL communication round is much longer than the BP convergence time.
Moreover, compared with the Quick-Init UCB algorithm, whose arm set is $\mathcal{A}$, the arm set in the BP-UCB algorithm is $\mathcal{N}$.
Similarly, we define the set of bad super arms by
\begin{equation}\label{bad_set}
\mathcal{K}_{\mathrm{B}}=\left\{\mathcal{K} \mid \mathbb{E}\left[{\sum\limits_{n = 1}^{K} {{r_{n}}\left( {{\boldsymbol{\omega} }\left( {{\mathcal{K}}} \right)} \right)} } \right] < \operatorname{opt}_{\boldsymbol{r}}\right\},
\end{equation}
where $\operatorname{opt}_{\boldsymbol{r}}=\mathbb{E}\left[r_n\left(\boldsymbol{\omega} \left(\mathcal{K}^{\ast}\right) \right)\right]$ is the optimal expected reward.
Besides, we denote the expected reward of arm $n$ as $\hat{r}(\mathcal{K}_n) = \mathbb{E}\left[{\sum_{n = 1}^{K} {{r_{t, n}}\left( {{\boldsymbol{\omega}}\left( {{\mathcal{K}_n}} \right)} \right)} } \right]$.
Then, we obtain the reward gaps between client $i$ and the best arm and the worst arm, respectively, i.e.,
\begin{equation}\label{best}
R_{\min }^{i}=\operatorname{opt}_{\boldsymbol{r}}-\max \left\{\hat{r}(\mathcal{K}_n) \mid \mathcal{K}_n \in \mathcal{K}_{\mathrm{B}}, i \in \mathcal{K}_n\right\},
\end{equation}
\begin{equation} \label{worst}
R_{\max}^{i}=\operatorname{opt}_{\boldsymbol{r}}-\min\left\{\hat{r}(\mathcal{K}_n) \mid \mathcal{K}_n \in \mathcal{K}_{\mathrm{B}}, i \in \mathcal{K}_n\right\}.
\end{equation}
Meanwhile, the best and worst regrets among $N$ clients in~\eqref{best} and~\eqref{worst} are given by $R_{\min }=\min _{i \in\mathcal{N}} R_{\min }^{i}$ and $R_{\max }=\max _{i \in\mathcal{N}} R_{\max }^{i}$, respectively.}

Then, we have the following theorem to upper bound the BP-UCB algorithm.
\begin{theorem}\label{Theorem02}
	\textit{Given a client selection problem with $N$ clients under the FL\&DA framework, when the communication budget is $K$,
		then the cumulative regret upper bound of~\eqref{reg} by running the BP-UCB algorithm over time horizon $T$ is
		\begin{equation}
			\begin{split}
				{\mathop{\rm Reg}\nolimits}\left( T \right) \le \frac{{8N{R_{\max }}\ln \left( T \right)}}{{{R_{\min }}^2}} + \left( {\frac{{{\pi ^2}}}{3} + 1} \right)NR_{\max }.
			\end{split}
	\end{equation}}
\end{theorem}
\begin{proof}
Please see Appendix C.
$\hfill\blacksquare$
\end{proof}

\begin{remark}
Theorem~\ref{Theorem02} shows that client selections in decentralized FA can achieve asymptotically diminishing regret under the condition of BP convergence shown in Theorem~\ref{Theorem03}.
Therefore, the proposed algorithm converges when the total time $T$ is sufficiently large.
\end{remark}

\section{Numerical Results} \label{Simulation Results}
In this section, several simulations are conducted to evaluate the proposed algorithms.
We consider a network scenario in a (1$\times$1) km$^2$ square area, where $N=20$ clients are randomly distributed within the coverage of the central server, as shown in Fig. \ref{network}.
The clients are self-organized in the FL\&DA framework, where each client links to its neighbor for message passing in the BP algorithm and the gossip method.
\com{The FL\&FA framework shares the same network with the  FL\&DA framework.
A slight difference between the FL\&FA and FL\&DA frameworks in Fig. \ref{network}  is that the former directly uploads their opinions to the central server, while the latter exchanges the opinions among the clients.}
At each time slot, only $K=5$ clients are selected to train the global model.
The channel vector $\boldsymbol{h}_{i}$ is modeled by the large-scale fading with the path-loss model: $\mathrm{PL[dB]\!=\!128.1+37.6}{\log _{10}}\left( d \right)$ \cite{xia2020multi}, where ${d}$ represents the Euclidean distance in km.
The transmit power ${P}_{0}$ is set to $\mathrm{20}$ dBm, and the average transmit SNR is $\mathrm{20}$ dB.
We assume that different clients have different data sizes.
To simulate the data quality heterogeneity, we divide the $20$ clients into two parts, i.e.,
clients $\{1,2,3,4,5 \}$ are with the i.i.d. data; while others are with non-i.i.d. data.
\begin{figure}[!t]
	\centering
	\includegraphics [width=2.8in]{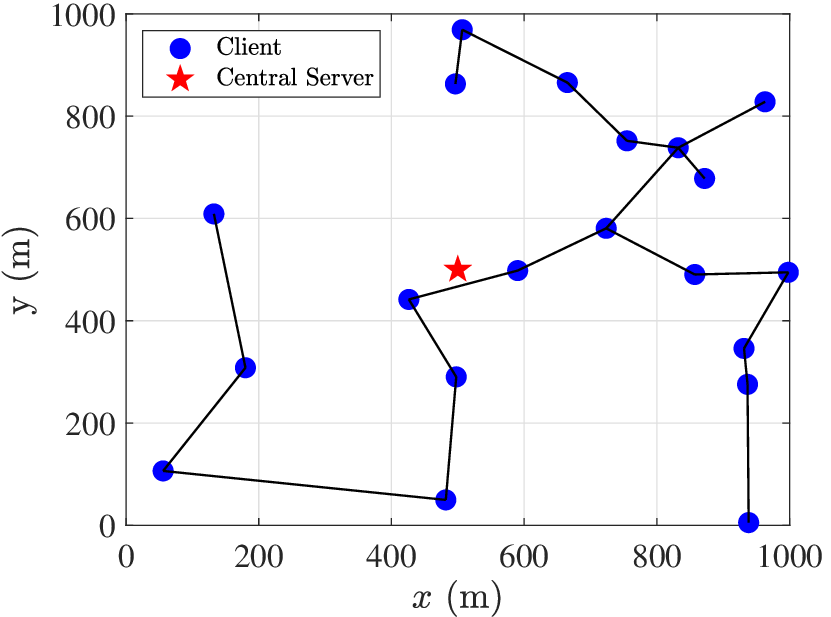}
	\caption{The network topology of a communication scenario in a (1$\times$1) km$^2$ square area, where $N=20$.}
	\label{network}
\end{figure}

We assume that the $N$ clients cooperatively train a classification model to classify the MNIST dataset using the support vector machine (SVM) method.
In the FL process, the total communication round $L = 15$ and the local epoch size $E=10$.
The batch size and the step size in the SGD method are $B=100$ and $\nu= 10^{-7}$, respectively.
\com{In the FA or DA process, we use the test data set to evaluate the trained global model.}
We consider the stochastically identical channel setting, where the successful transmission probability is set to $\theta_n = 1, \ \forall n \in \mathcal{N}$.
In addition, the learning rates in the Quick-Init UCB and BP-UCB algorithms are $\mu = 1$ and $\mu = 0.01$, respectively.
The BP algorithm sets the large-scale shadowing factors $C$ and $\beta$ to $-30$ dB and $3.7$, respectively.
In the following, all numerical results are obtained from 100 Monte Carlo trials.
\begin{figure*}[!t]
\centering
\subfloat{\includegraphics[width=2.2in]{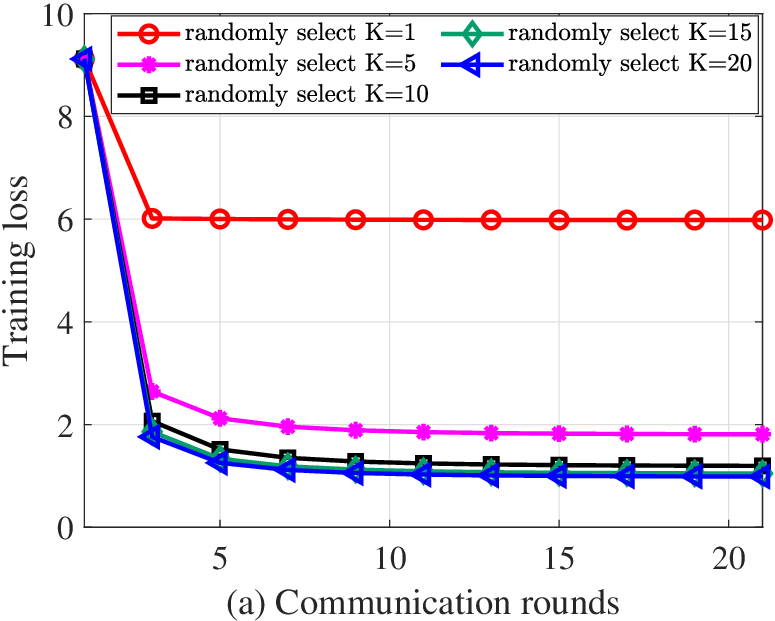}%
\label{Lgradual_case1}}
\hfil
\subfloat{\includegraphics[width=2.2in]{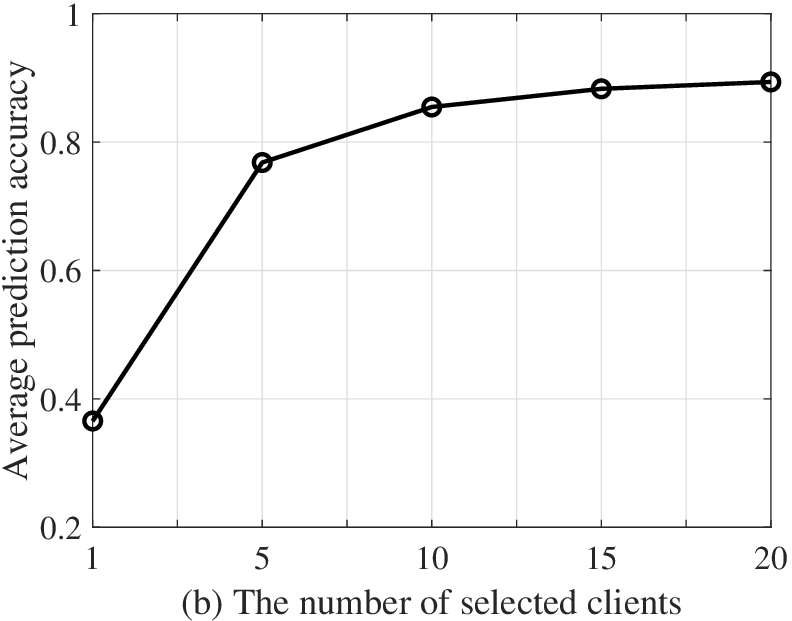}
\label{Lsteep_case2}}
\hfil
\subfloat{\includegraphics[width=2.2in]{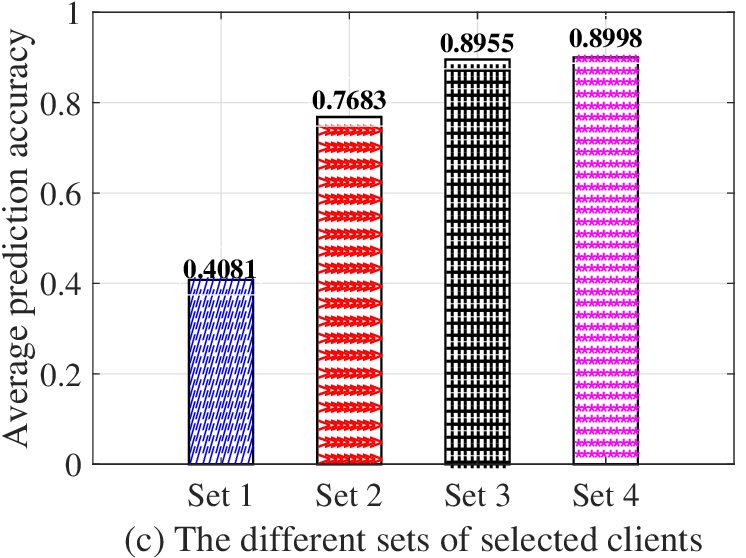}
\label{Llossy_case3}}
\caption{(a) The training loss of different numbers of selected clients; (b) The average prediction accuracy of the different numbers of selected clients; (c) The average prediction accuracy of different sets of selected clients.}
\label{1_2}
\end{figure*}

\subsection{Under the FL\&FA Framework}
We first investigate the impacts of different client sets and the number of participant clients on the trained global model by using the random selection method.
The random selection method uniformly selects $K$ out of $N$ clients at each time slot.
Fig.~\ref{1_2}(a)  shows the training loss in the FL process with the different number of participant clients.
We see that all client groups can converge when the communication round is over $15$.
The more participant clients in the FL process, the better performance of the trained global model can be achieved.
The same result can be observed in Fig.~\ref{1_2}(b), which shows the time-to-accuracy (i.e., the average prediction accuracy) of the trained global model with the different participant clients.
We see that the performance is not significantly improved when the number of participant clients is over $15$.
Therefore, it is necessary to strike a balance between the accuracy of the trained global model and the communication budget of $K$.
Fig.~\ref{1_2}(c) shows the time-to-accuracy of different combinations of client sets: set 1 $\rightarrow$ $\{16, 17, 18, 19, 20\}$,
set 3 $\rightarrow$ $\{1, 2, \ldots, 20\}$, set 4 $\rightarrow$ $\{1, 2, 3, 4, 5\}$, and set 2 consists of 5 random selected clients.
We see that client set 4 exhibits the best performance and is approximate to all selected cases.
By contrast, client set 1 accounts for the worst performance, even worse than the random selection method.
To sum up, Fig.~\ref{1_2} demonstrates that it is necessary to perform the client selection in the proposed framework, and the client set 4 can be regarded as the optimal solution of problem \eqref{target}.
\begin{figure}[!t]
	\centering
	\includegraphics [width=2.8in]{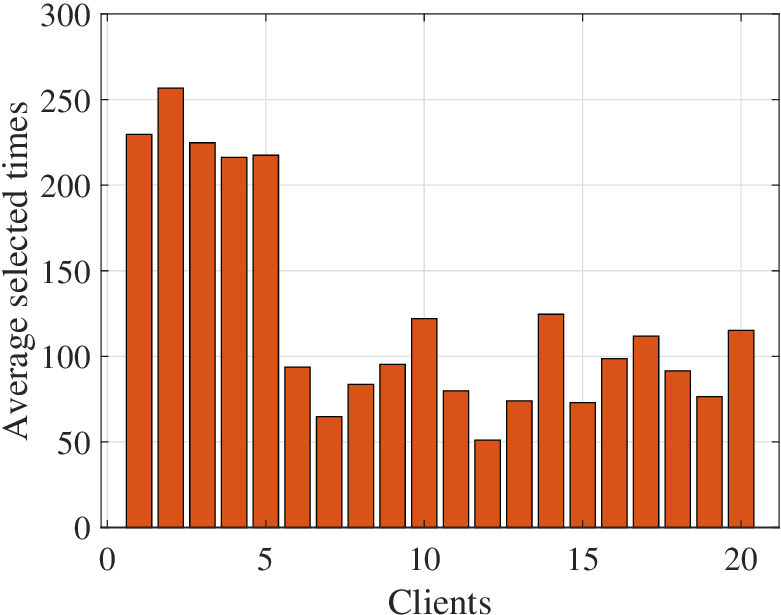}
	\caption{The average selected times of the participant clients by running the Quick-Init UCB algorithm where $N=20$ and $K=5$.}
	\label{cen_selected_time}
\end{figure}

In the FL\&FA framework, the client's local model parameters and test opinion are required to upload to the central server.
Fig.~\ref{cen_selected_time} shows \com{the average selected time of}different clients by running the Quick-Init UCB algorithm.
We see that the client set $\{1, 2, 3, 4,5\}$ has the highest selection frequency, leading to a better performance of the trained global model.
The average selected times of the client set $\{1, 2, 3, 4,5\}$ is $228.95$, which is more than $153\%$ compared with $90.34$ in the non-i.i.d. client group.
This indicates that the Quick-Init UCB algorithm can successfully find the optimal client set.

\begin{figure}[!t]
	\centering
	\includegraphics [width=2.8in]{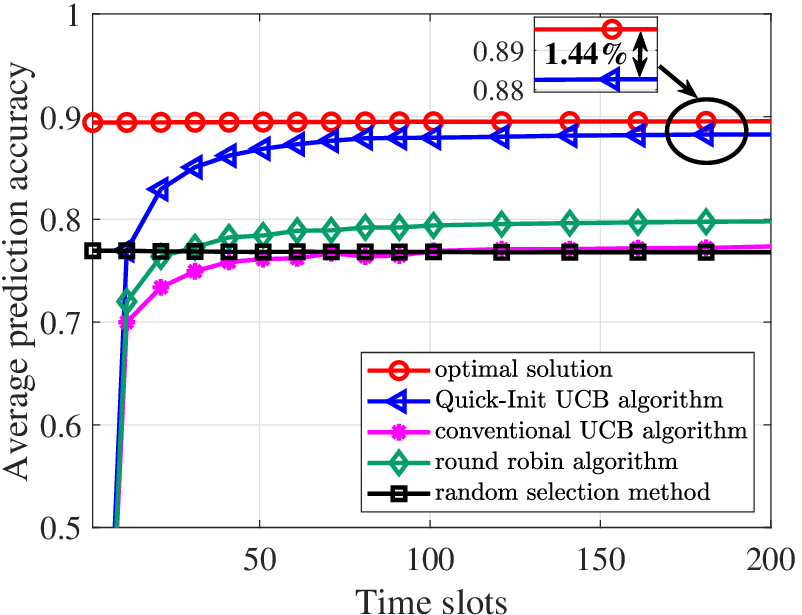}
	\caption{The average prediction accuracy of different selection methods via time slot $t$ under the FL\&FA framework, where $N=20$ and $K=5$.}
	\label{cen_performance}
\end{figure}
Fig.~\ref{cen_performance} compares the time-to-accuracy of the optimal solution, Quick-Init UCB algorithm, conventional UCB algorithm, round-robin algorithm, and the random selection method under the centralized FL$\&$FA framework.
The conventional UCB algorithm is similar to the Quick-Init UCB algorithm except for lacking the quick initiation phase.
The round-robin algorithm selects the predefined client sets sequentially.
Here, we define five client sets for the round-robin algorithm, i.e., the client set $\{1, 2, 11, 16, 20\}$, and clients in the other four sets are chosen randomly.
From Fig.~\ref{cen_performance}, we see that the prediction accuracy of the Quick-Init UCB algorithm increases from $76.62\%$ to $86.57\%$ compared with the random selection method and is about  $1.44\%$ less than the optimal solution.
Besides, the conventional UCB algorithm is much worse than the Quick-Init UCB algorithm, which performs similarly to the random selection method.
The reason is that the initiation duration (i.e., about $15,504$ time slots) for the conventional UCB algorithm is much longer than the time horizon $T=200$.
This demonstrates that the Quick-Init UCB algorithm can efficiently initiate all the arms in a few time slots and find the optimal solution quickly.

\subsection{Under the FL\&DA Framework}
\begin{figure}[!t]
	\centering
	\includegraphics [width=2.8in]{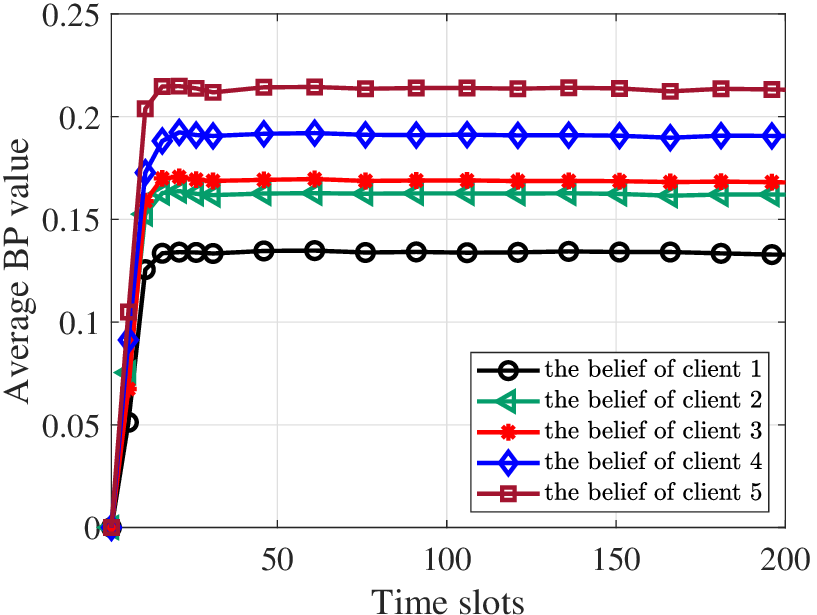}
	\caption{The average BP values of first $K$ clients via time slot $t$ by running the BP-UCB algorithm, where $N=20$ and $K=5$.}
	\label{BP_value}
\end{figure}
Next, we study the convergence of the BP algorithm and evaluate the BP-UCB algorithm under the FL\&DA framework.
Fig.~\ref{BP_value} shows the average BP value of client $\{1,2,3,4,5\}$ by running the BP-UCB algorithm.
We see that all five clients' belief values calculated by \eqref{Belief}  converge fast, which is about $10$ time slots.
Besides, different clients have different beliefs, reflecting the system and data heterogeneity of the client.
In fact, the better the client's belief value, the higher the probability it will be selected.
This is also why we incorporate the belief value into constructing the UCB index.

\begin{figure}[!t]
	\centering
	\includegraphics [width=2.8in]{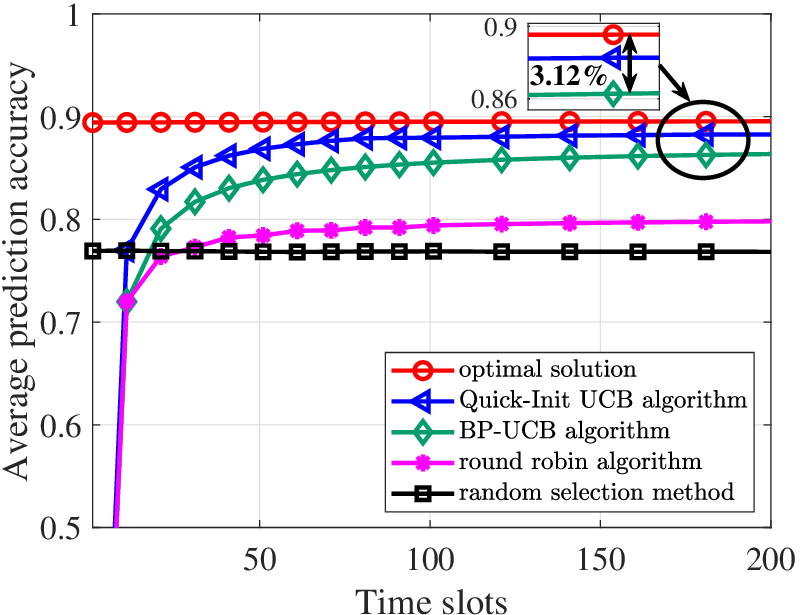}
	\caption{The average prediction accuracy of different selection methods via time slot $t$ under the FL\&DA framework, where $N=20$ and $K=5$.}
	\label{decen_performance}
\end{figure}
Fig.~\ref{decen_performance} compares the performance of the optimal solution,  BP-UCB algorithm,  Quick-Init UCB algorithm,  round-robin algorithm, and the random selection method under the FL$\&$DA framework.
The exploration factor $\mu$ in the BP-UCB algorithm equals $0.01$.
It can be seen from Fig.~\ref{decen_performance} that the prediction accuracy of the BP-UCB algorithm is improved from $72.02\%$ to $86.28\%$ compared with the random selection method.
In addition, the performance gap between the optimal solution and the BP-UCB algorithm is less than $3.12\%$.
More importantly, the performance of the BP-UCB algorithm is slightly worse than the Quick-Init UCB algorithm.
However, the BP-UCB algorithm does not require uploading the client's opinion to the central server.

\begin {table*}[!t]\com{
\centering \caption{Time-to-accuracy of different client selection algorithms under the FL\&FA and FL\&DA frameworks, where $T=500$.}
\begin {tabular}{c| c |c| c| c | c |c | c |c}
\toprule	
Algorithms:
 &\multicolumn{2}{c|}{\makecell[c]{{Optimal}\\ Solution}} & \multicolumn{2}{c|}{\makecell[c]{Random\\Selection}} &\multicolumn{2}{c|}{\makecell[c]{Quick-Init UCB\\Algorithm}} & \multicolumn{2}{c}{\makecell[c]{BP-UCB\\ Algorithm}} \\ \midrule
target accuracy & time slots & test accuracy &time slots &test accuracy &time slots &test accuracy &time slots &test accuracy \\ \midrule
0.75 & 6.00 & 0.7683 & 40.24 & 0.7519 & 9.92 & 0.7550 & 14.09 & 0.7541 \\ \midrule
0.78 & 7.00 & 0.7843 & 391.24 & 0.7752 & 13.55 & 0.7832 & 23.63 & 0.7824 \\ \midrule
0.82 & 11.00 & 0.8217 & 500.00 & $-$ & 26.42 & 0.8209 & 57.95&0.8196 \\ \midrule%\cmidrule {2-3}
0.85 & 19.00 & 0.8515 & 500.00 & $-$ & 111.22 & 0.8494 &131.81&0.8466 \\
\bottomrule
\end{tabular}\label{comparison_table}}
\end {table*}					
Next, we compare the time-to-accuracy of different selection algorithms under the FL\&FA and FL\&DA frameworks, as shown in Table~\ref{comparison_table}, where the total time slots are $T=500$, and the target accuracy is set to $\{0.75, 0.78, 0.82, 0.85\}$.
Symbol ``$-$'' represents that the selection algorithm cannot achieve the target accuracy within $500$ time slots.
It can be seen that the optimal solution consumes the least time slots compared with other methods to achieve the target accuracy.
In contrast, the random selection method accounts for the worst performance, which cannot achieve the target accuracy in the cases of $0.82$ and $0.85$.
In addition, the time-to-accuracy of the BP-UCB algorithm is slightly worse than the Quick-Init UCB algorithm.
Both can achieve the target accuracy and are close to the optimal solution.
These results are consistent with those shown in Fig \ref{decen_performance}.

\begin{figure}[!t]
	\centering
	\includegraphics [width=2.8in]{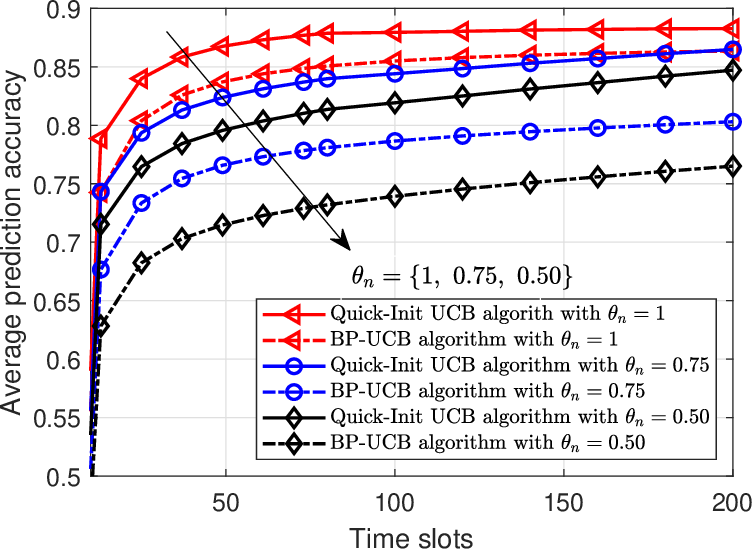}
	\caption{\com{The average prediction accuracy of the Quick-Init UCB algorithm and the BP-UCB algorithm via the successful transmission probability $\theta_n=\{ 1, 0.75, 0.50\}$, where $N=20$ and $K=5$.}}
	\label{PreTheta}
\end{figure}
\com{Fig. \ref{PreTheta} depicts the average prediction accuracy of two algorithms, Quick-Init UCB and BP-UCB, across different successful transmission probabilities $\theta_n=\{ 1, 0.75, 0.5\}, \ \forall n \in \mathcal{N}$, where $N=20$, $K=5$, and $T=200$.
We see that the Quick-Init UCB algorithm consistently outperforms the BP-UCB algorithm across various successful transmission probabilities. This superiority can be attributed to the Quick-Init UCB algorithm's ability to leverage global information of the trained model, made possible by the centralized FL\&FA framework.
Moreover, as the successful transmission probability $\theta_n$ decreases from $1$ to $0.5$, both proposed algorithms experience a decline in average prediction accuracy. This decline can be attributed to the impact of the successful transmission probability on the total volume of data involved in the federated training process. Consequently, the reduction in average prediction accuracy is more pronounced for the BP-UCB algorithm, which requires more communication rounds to converge than the BP-UCB algorithm.}

\begin{figure}[!t]
	\centering
	\includegraphics [width=2.8in]{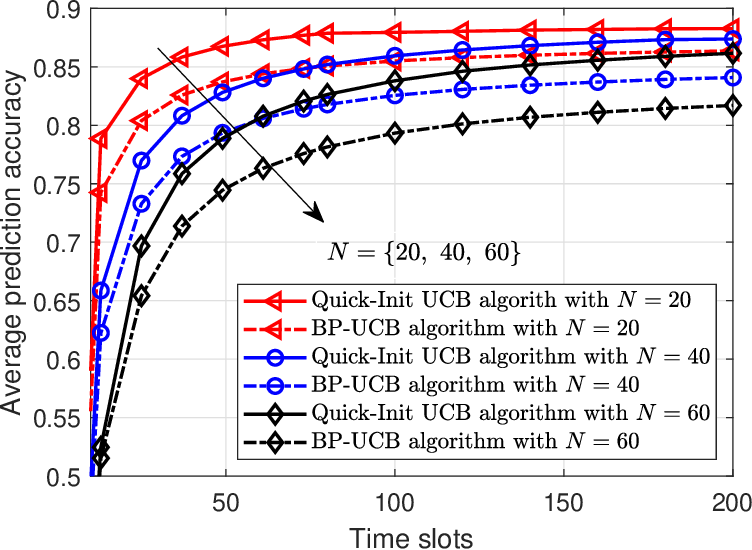}
	\caption{\com{The average prediction accuracy of the Quick-Init UCB algorithm and the BP-UCB algorithm via the number of clients $N=\{ 20, 40, 60\}$, where $\theta_n = 1$ and $K=5$.}}
	\label{PreNum}
\end{figure}
\com{Fig. \ref{PreNum} compares the average prediction accuracy of the Quick-Init UCB algorithm and the BP-UCB algorithm via the number of users $N=\{ 20, 40, 60\}$, where $\theta_n = 1$, $K=5$ and $T=200$.
We see that the Quick-Init UCB algorithm's performance is better than the BP-UCB algorithm under different numbers of clients $N=\{ 20, 40, 60\}$.
In addition, the average prediction accuracy of the Quick-Init UCB and BP-UCB algorithms decreases with the number of clients.
This can be attributed to the increasing heterogeneity of user data and the growing complexity of the system, which expands the algorithm's exploration space in searching of the optimal set of $K=5$ clients. Consequently, this expansion leads to a reduction in the average prediction accuracy of the proposed algorithms.}

\begin{figure}[!t]
	\centering
	\includegraphics [width=2.8in]{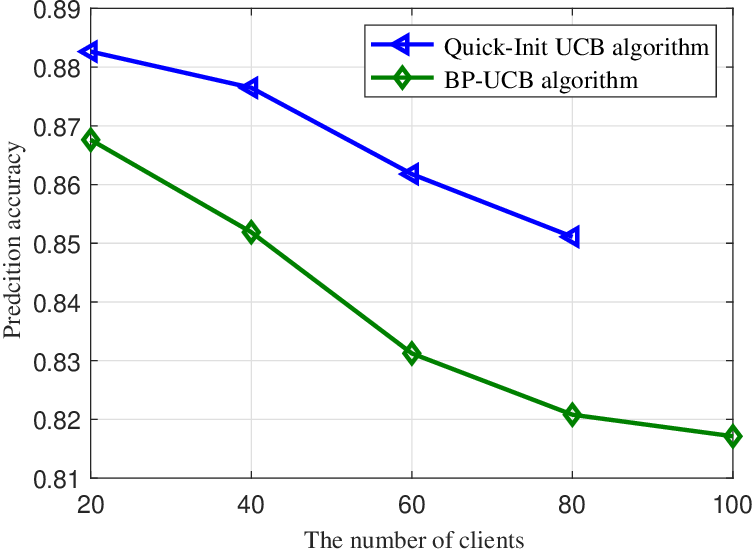}
	\caption{\com{The prediction accuracy of the Quick-Init UCB algorithm and the BP-UCB algorithm via the number of clients $N=\{ 20, 40, 60, 80, 100\}$, where $\theta_n=1$, $K=5$, and $T=500$.}}
	\label{Pre2Slot}
\end{figure}
\com{Fig.~\ref{Pre2Slot} shows the prediction accuracy of the Quick-Init UCB and  BP-UCB algorithms under the FL\&FA and FL\&DA frameworks across different numbers of clients $N=\{20, 40, 60, 80, 100\}$, where $\theta_n = 1$, $K=5$, and $T=500$.
We see that the performance of the Quick-Init UCB is better than the BP-UCB algorithm.
However, it is worth noting that the Quick-Init UCB algorithm suffers from high computational complexity due to the increasing exploration space with the number of clients. For example, when $N=100$, the exploration space becomes $|\mathcal{A}| = \binom{100}{5}>10^7$, which exceeds the computer memory capacity, making it impossible to plot the results of the Quick-Init UCB algorithm.
While the Quick-Init UCB algorithm demonstrates superior prediction performance, its high computational complexity limits its practical applicability. On the other hand, the BP-UCB algorithm exhibits better scalability but poorer prediction performance. Therefore, Figs. \ref{PreNum} and \ref{Pre2Slot} highlight the tradeoff between prediction accuracy and computational complexity when implementing the proposed algorithms in practical scenarios.}

\com{
\subsection{Comparative Analysis}
At last, we provide a detailed comparative analysis to assess the scalability and adaptability of the proposed algorithms and the baseline algorithms based on the above simulations.
At each time slot, the round-robin algorithm sequentially selects predefined client sets, while the random selection algorithm randomly chooses a set from the predefined client sets. Thus, the round-robin and random selection algorithms exhibit superior scalability and adaptability as they do not rely on any information from the evaluation process.
However, we see from Figs. \ref{cen_performance} and  \ref{decen_performance} that they achieve the lowest prediction accuracy.

Both Quick-Init UCB and UCB algorithms need to compute the UCB indices using the evaluation results.
In contrast, the Quick-Init UCB algorithm needs to improve its computational complexity due to the exponential size of the client set space, thereby limiting its scalability. To address this issue, the BP-UCB algorithm adopts a distributed network structure where the clients are self-organized. This approach significantly reduces the feasible client sets from $\binom{N}{K}$ to $N$, where $N$ and $K$ represent the total number of clients and communication budgets, respectively. Nevertheless, the adaptability of the BP-UCB algorithm is inferior to that of the Quick-Init UCB algorithm. This is attributed to the requirement that clients are self-organized, which is not frequently observed in practical applications.
In summary, the BP-UCB algorithm demonstrates slightly lower performance than the Quick-Init UCB algorithm, as illustrated in Figs. \ref{PreTheta}-\ref{Pre2Slot}, it offers a better scalability.}

\section{Conclusion} \label{Conclusion}
This paper proposed the FL\&FA and FL\&DA frameworks to improve the efficacy of the trained global model.
Based on these frameworks, we studied the \textit{goal-directed} client selection problem in an FL-enabled communication network by considering the clients' system and data heterogeneity subject to limited communication resources.
By formulating it as an SMAB problem, we put forth the Quick-Init UCB algorithm and the BP-UCB algorithm to select the optimal client subset for model training under the FL\&FA and FL\&DA frameworks, respectively.
Numerical results show that the BP-UCB algorithm has a comparable performance with the Quick-Init UCB algorithm, even though the BP-UCB algorithm does not require uploading any information to the central server.
In addition, we derived two regret upper bounds (i.e., both with $\mathcal{O}(\mathrm{ln}(T))$) for the Quick-Init UCB algorithm and the BP-UCB algorithm, respectively, which both increase logarithmically with time slot $T$.
Moreover, we conducted extensive simulations to evaluate the Quick-Init UCB and BP-UCB algorithms.
Numerical results show that the proposed algorithms achieve nearly optimal performance, with a marginal gap of less than $1.44\%$ and $3.12\%$, respectively.

\appendices
\section{Proof of Theorem 1}\label{appendix1}
Compared with the conventional UCB algorithm, the Quick-Init UCB algorithm has an extra quick initialization phase.
This quick initialization phase can accelerate the convergence rate of the UCB algorithm.
However, it will not change the time-increase trend of the regret upper bound in the conventional UCB algorithm.
As a result, we can derive the regret upper bound of the Quick-Init UCB algorithm based on the finite-time regret analysis in  \cite{auer2002finite}.
Formally, we define the performance gap between the optimal expected reward and reward expectation for arm $a$ by
\begin{equation} \label{each0}
\delta_{a}=\operatorname{opt}_{\boldsymbol{r}}-\mathbb{E}\left[r_{a}(\mathcal{K}_a)\right], \forall a \in \mathcal{A}.
\end{equation}
Then, according to \cite{auer2002finite},  the regret upper bound of the Quick-Init UCB algorithm can be written as
\begin{equation}\label{app0_1}
{\mathop{\rm Reg}\nolimits}\left( T \right) \le
 \left[8 {\sum_{a:a\in \mathcal{K}_B}}\left(\frac{\ln T}{\delta_{a}}\right)\right]+\left(\frac{\pi^2}{3}+1\right)\left(\sum_{a=1}^{\left|\mathcal{A} \right|} \delta_{a}\right).
\end{equation}

Let
\begin{equation}\label{best0}
R_{\min }=\operatorname{opt}_{\boldsymbol{r}}-\max \left\{\hat{r}(\mathcal{K}_a) \mid \mathcal{K}_a \in \mathcal{K}_{\mathrm{B}}\right\},
\end{equation}
\begin{equation} \label{worst0}
R_{\max}=\operatorname{opt}_{\boldsymbol{r}}-\min\left\{\hat{r}(\mathcal{K}_a) \mid \mathcal{K}_a \in \mathcal{K}_{\mathrm{B}}\right\}.
\end{equation}
be the worst and best reward gaps between the best and worst arm compared with the optimal super arms, respectively.
By combining with \eqref{best0} and \eqref{worst0}, we have
\begin{equation}
\begin{aligned} \label{bound0}
{\mathop{\rm Reg}\nolimits}\left( T \right) \leq
& \left[8 {\sum_{a:a\in \mathcal{K}_B}}\left(\frac{\ln T}{\delta_{a}}\right)\right]+\left(\frac{\pi^2}{3}+1\right)\left(\sum_{a=1}^{\left|\mathcal{A} \right|} \delta_{a}\right)\\
\leq & \frac{{8{\left|\mathcal{A} \right|}\ln \left( T \right)}}{{{R_{\min }}}} + \left( {\frac{{{\pi ^2}}}{3} + 1} \right){\left|\mathcal{A} \right|}R_{\max }.\\
\end{aligned}
\end{equation}
This concludes the proof.
$\hfill\blacksquare$

\section{Proof of Theorem 2}\label{appendix3}
According to \cite{mooij2005sufficient}, the BP update-equation with binary variables can be written as
\begin{equation}
	\tanh \tilde{\nu}^{j \rightarrow i}=\tanh \left(J_{i j}\right) \tanh \left(\theta_{j}+\sum_{{k \in \mathrm{Neighbor}(j)/ i}} \nu^{k \rightarrow j}\right),
	\label{BPrewrite}
\end{equation}
where $\tilde{\nu}^{j \rightarrow i}$ is the iterative message from neighboring client $j$ to $i$ and $\nu^{k \rightarrow i}$ is the message from neighboring client $k$ to $i$.
Symbols $J_{i j}$ and $\theta_{j}$ are the coupling between client $i$ and $j$ and the local field in client $j$, respectively.
In fact, the parallel BP update can be viewed as a mapping function of clients $i$ and $j$, i.e., $\tilde{\nu}_i = f\left( \nu_j \right)$, which is specified in \eqref{BPrewrite}.
Specifically, the incoming messages $\nu$ to client $j$ are mapped to a message output $\tilde{\nu}$ to client $i$.
Therefore, we need to derive the conditions that the mapping function $f$ is contraction so that the message converges to a fixed point.

Without loss of generality, we adopt the theorem in \cite{mooij2005sufficient} to show the function contraction.
For continence, we copy it as follows:

\textit{Let $f$ : $X\to X$ be a contraction of a complete metric space $\left( {{\rm{X, d}}} \right)$.
Then $f$ has a unique fixed point ${x_\infty } \in X$ and for any $x \in X$, the sequence $\{x, f\left( x \right), {f^2}\left( x \right), \ldots\}$  obtained by iterating $f$ converges to $x_\infty$.
The rate of convergence is at least linear since $d(f(x), {x_\infty }) \le Kd(x, {x_\infty })$ for all $x \in X$.}

Based on this theorem, the finite-dimensional vector space $V$ can be constructed by the following lemma:

\textit{If ${\sup _{\nu \in V}}\left\| {{f'}\left( \nu \right)} \right\| < 1$, then $f$ is a $\left\|  \cdot  \right\|$-contraction and the consequence of the above theorem  holds}.

Then, we apply this lemma to obtain the derivation of $f$, which can be calculated by \eqref{BPrewrite}, i.e.,
\begin{equation}\label{derivation}
	\left(f^{\prime}(\nu)\right)_{j \rightarrow i, k \rightarrow l}=\frac{\partial \tilde{\nu}^{j \rightarrow i}}{\partial \nu^{k \rightarrow l}}=G_{j \rightarrow i, k \rightarrow l} A_{j \rightarrow i}(\nu),
\end{equation}
where
\begin{equation}\label{G}
	G_{j \rightarrow i}(\nu):=\frac{1-\tanh ^{2}\left(\theta_{j}+\sum_{k \in \partial j \backslash i} \nu^{k \rightarrow j}\right)}{1-\tanh ^{2}\left(\tilde{\nu}^{j \rightarrow i}(\nu)\right)} \operatorname{sgn} J_{i j}
\end{equation}
\begin{equation}\label{A}
	A_{j \rightarrow i, k \rightarrow l}:=\tanh \left|J_{i j}\right| \delta_{j, l} \mathbf{1}_{{\mathrm{Neighbor}(j)/ i}}(k),
\end{equation}
where all $\nu$-dependent variables are absorbed in \eqref{G}, and \eqref{A} is dependent on the coupling between neighbor $i$ and $j$ and the pairwise interactions.
Although term $A_{j \rightarrow i, k \rightarrow l}$ is defined by the pair clients $i$ and $j$ as well as pair clients $k$ and $l$ in the BP network, only the linked client $l$ with $j$ affects the message passing from client $j$ to client $i$.
Note that ${\sup _{\nu \in V}}\left\| {G_{j \rightarrow i}(\nu)} \right\| = 1$, which indicates that
\begin{equation}\label{derivation_upper}
	\left|\frac{\partial \tilde{\nu}^{j \rightarrow i}}{\partial \nu^{k \rightarrow l}}\right| \leq A_{j \rightarrow i, k \rightarrow l}
\end{equation}
is satisfied on $V$.
Therefore, the upper bound of the derivation of $f$  can be computed by
\begin{equation}
	\begin{aligned}
		\left\|f^{\prime}(\nu)\right\|_{1} & \leq \max _{k \rightarrow l} \sum_{j \rightarrow i} \tanh \left|J_{i j}\right| \delta_{j l} \mathbf{1}_{{ \mathrm{Neighbor}(j)/ i}}(k) \\
		&=\max _{l \in V} \max _{{k \in \mathrm{Neighbor}(l)}} \sum_{{i \in \mathrm{Neighbor}(l)/ k}} \tanh \left|J_{i l}\right| .
	\end{aligned}
\end{equation}
Based on the above lemma, we obtain
\begin{corollary} 	\label{cor-1}
	\textit{If \begin{equation}
			\max _{l \in V} \max _{k \in \mathrm{Neighbor}(l)} \sum_{{i \in \mathrm{Neighbor}(l)/ k}} \tanh \left|J_{i l}\right|<1,
		\end{equation}
		BP is the contraction mapping function and converges to a unique value, irrespective of the initial messages.}
\end{corollary}
Consequently, the convergence of the BP algorithm depends on the summation of couplings under the neighboring constraints.

We rewrite the  compatibility function between clients $n$ and $i$ in the following
\begin{equation}
	{\psi _{i,n}}\left( {{\mathbb{I}^{t}_i},{\mathbb{I}^{t}_n}} \right) = \exp \left( { - Cd_{{i},{n}}^\beta } \right),
	\label{compa}
\end{equation}
where the coupling is computed by ${J_{i,l}} =  - Cd_{i,l}^\beta $, where $C$ and $\beta$ are two fading-relative constants and $d_{i,l}$ denotes the distance between clients $i$ and $l$.
Therefore, by applying \eqref{compa} to \textbf{Corollary}~\ref{cor-1}, we obtain
\begin{equation}
	\begin{aligned}
		&\max _{l \in V} \max _{k \in \mathrm{Neighbor}(l)} \sum_{{i \in \mathrm{Neighbor}(l)/ k}} \tanh \left|J_{i l}\right| \\
		&= \max _{l \in V} \max _{k \in \mathrm{Neighbor}(l)} \sum_{{i \in \mathrm{Neighbor}(l)/ k}} \tanh \left|d_{i l}^\beta \right|  \\
		&\le \left( {N - 1} \right)\tanh \left| {d_{\max }^\beta} \right|. \\
	\end{aligned}
\end{equation}
where $d_{\max }$ denotes the maximum distance between two arbitrary clients in the BP network.
Therefore, if $\left( {N - 1} \right)\tanh \left| {d_{\max }^\beta} \right| < 1$ strictly holds in the network, BP will converge, which concludes this proof.
$\hfill\blacksquare$

\section{Proof of Theorem 3}\label{appendix4}
The regret ${\mathop{\rm Reg}\nolimits}\left( T \right)$ after $T$ time slots can be written as
\begin{equation} \label{cenReg_1}
{\mathop{\rm Reg}\nolimits}\left( T \right) = \sum_{\mathcal{K} \in \mathcal{K}_{\mathrm{B}}} \Delta(\mathcal{K}) \mathbb{E}\left[\Pi_{T,\mathcal{K}}\right],
\end{equation}
where $\Delta(\mathcal{K}) = \operatorname{opt}_{\boldsymbol{r}}  -\hat{r}(\mathcal{K})$.
Since the number of times that client $k$ is selected in set $\mathcal{K}$ is less than the times that client $k$ is selected in other sets after $T$ time slots, we obtain that $\Pi_{T,\mathcal{K}} \le \Pi_{T,k}$ \cite{chen2013combinatorial}, where $k \in \mathcal{K}$.
Then, we can bound the regret in \eqref{cenReg_1} by
\begin{equation} \label{cenReg_2}
\begin{aligned}
{\mathop{\rm Reg}\nolimits}\left( T \right) &= \sum_{\mathcal{K} \in \mathcal{K}_{\mathrm{B}}} \Delta(\mathcal{K}) \mathbb{E}\left[\Pi_{T,\mathcal{K}}\right] \\
&\le NR_{\max } \mathbb{E}\left[\Pi_{T,i}\right], i \in \mathcal{N}.
\end{aligned}
\end{equation}
It can be seen that bounding \eqref{cenReg_2} is equal to bounding the average number of times that each client be selected, i.e., $\mathbb{E}\left[\Pi_{T,i}\right]$.

Let $c_{t,{\Pi_{t,i}}} = \mu \sqrt {{{\ln \left( t \right)}}/{{{\Pi_{t,i}}}}}$, for any client $i$.
In the following, we bound ${\pi_{t,i}}$ for the number of times that client $i$ is selected after $T$ time slots. In the initialization phase, each client $i$ is selected once. Therefore, we obtain the following upper bound \cite{auer2002finite},
\begin{equation}
\begin{aligned} \label{bound_number}
\Pi_{T,i} &=1+\sum_{t=t_0+1}^{T}\left\{{\mathbb{I}^{t}_i}\right\} \\
& \leq \ell+\sum_{t=t_0+1}^{T}\left\{{\mathbb{I}^{t}_i}, \Pi_{t-1,i} \geq \ell\right\} \\
& \leq \ell+\sum_{t=t_0+1}^{T}\left\{\bar{r}_{\Pi^{*}(t-1)}^{*}+c_{t-1, \Pi^{*}(t-1)} \leq \bar{r}_{i,\Pi_{t-1,i}}\right.\\
&\left.+c_{t-1, \Pi_{t-1,i}}, \Pi_{t-1,i} \geq \ell\right\} \\
\leq & \ell+\sum_{t=t_0+1}^{T}\left\{\min _{0<s<t} \bar{r}_{s}^{*}+c_{t-1, s} \leq \max _{\ell \leq s_{i}<t} \bar{r}_{i, s_{i}}+c_{t-1, s_{i}}\right\} \\
& \leq \ell+\sum_{t=1}^{\infty} \sum_{s=1}^{t-1} \sum_{s_{i}=\ell}^{t-1}\left\{\bar{r}_{s}^{*}+c_{t, s} \leq \bar{r}_{i,s_{i}}+c_{t, s_{i}}\right\}.
\end{aligned}
\end{equation}
The last term in \eqref{bound_number} indicates that $\bar{r}$  depends on the confidence interval, which can be constructed using the Chernoff-Hoeffding inequality.

Based on the above fact, we obtain $\mathbb{P}\left\{\bar{r}_{s}^{*} \leq r^{*}-c_{t, s}\right\} \leq e^{-2\mu^2 \ln t}=t^{-2\mu^2}$ and $\mathbb{P}\left\{\bar{r}_{i,s_{i}} \ge r_{i}^{*} + c_{t, s_{i}}\right\} \leq e^{-2\mu^2 \ln t}=t^{-2\mu^2}$, where $r_{i}^{*}$ denotes the expected value of the probability distribution in client $i$'s reward.
Let $\ell=\left\lceil(8 \ln T) / R_{i}^{2}\right\rceil$, we have
\begin{equation}
\begin{aligned} \label{bound}
\mathbb{E}\left[\pi_{T,i}\right] \leq & {\left[\frac{8 \ln T}{R_{i}^{2}}\right]+\sum_{t=1}^{\infty} \sum_{s=1}^{t-1} \sum_{s_{i}=\left\lceil(8 \ln n) / R_{i}^{2}\right\rceil}^{t-1} } \\
& \times\left(\mathbb{P}\left\{\bar{r}_{s}^{*} \leq r^{*}-c_{t, s}\right\} +\mathbb{P}\left\{\bar{r}_{i,s_{i}} \ge r_{i}^{*} + c_{t, s_{i}}\right\}\right) \\
\leq &\left\lceil\frac{8 \ln T}{R_{i}^{2}}\right\rceil+\sum_{t=1}^{\infty} \sum_{s=1}^{t} \sum_{s_{i}=1}^{t} 2t^{-2\mu^2}. \\
\end{aligned}
\end{equation}
When $\mu^{2} \geq 3 / 2$, we have $\sum_{t=1}^{\infty} t^{-2 \mu^{2}+1} \leq \sum_{t=1}^{\infty} t^{-2}=\pi^{2} / 6$.
Equation \eqref{bound} can be bounded as
\begin{equation}
\begin{aligned} \label{bound1}
\mathbb{E}\left[\pi_{T,i}\right] \leq &\left\lceil\frac{8 \ln T}{R_{i}^{2}}\right\rceil+\sum_{t=1}^{\infty} \sum_{s=1}^{t} \sum_{s_{i}=1}^{t} 2t^{-2\mu^2} \\
\le & \frac{8 \ln T}{R_{\min }^{2}} + 1 + \frac{\pi^{2}}{3}. \\
\end{aligned}
\end{equation}
Denote the convergence value of BP by ${b_{\infty ,i}}$ in client $i$ for UCB selection.
We obtain $\mathbb{P}\left\{\bar{r}_{s}^{*} \leq r^{*}-c_{t, s}-{b_{\infty ,i}}\right\} \leq 2 e^{-\frac{2(c+b)^{2} {\pi}^{2}}{\pi^{2}} \ln t} \leq e^{-2\mu^2 \ln t}=t^{-2\mu^2}$ and $\mathbb{P}\left\{\bar{r}_{i,s_{i}} \ge r_{i}^{*} + c_{t, s_{i}}\right\} \leq 2 e^{-\frac{2(c+b)^{2} {\pi}^{2}}{\pi^{2}} \ln t}\leq e^{-2\mu^2 \ln t}=t^{-2\mu^2}$ from Chernoff-Hoeffding inequality.
By summating over all agents, we conclude the proof.
$\hfill\blacksquare$

\balance
\bibliography{BibDIRP}
\bibliographystyle{IEEEtran}
\end{document}